\documentclass{article}

\PassOptionsToPackage{numbers,sort&compress,square}{natbib}

\usepackage[final,main]{neurips_2026}
\usepackage[T1]{fontenc}

\usepackage[utf8]{inputenc}
\usepackage{microtype}

\usepackage{graphicx}
\usepackage{subcaption}
\usepackage{booktabs}

\usepackage{hyperref}
\hypersetup{hidelinks,
  pdftitle={Robust and Hard-to-Remove GNN Watermarking via Topological Invariant Perception},
  pdfauthor={Jipeng Li and Yanning Shen}}

\usepackage{amsmath}
\usepackage{amssymb}
\usepackage{mathtools}
\usepackage{amsthm}
\usepackage[capitalize,noabbrev]{cleveref}
\usepackage{xurl}

\usepackage[table]{xcolor}
\usepackage{wrapfig}
\usepackage{multirow}
\usepackage{enumitem}
\usepackage{needspace}

\usepackage{makecell}
\renewcommand{\arraystretch}{1.18}
\newcommand{\accwm}[2]{\makecell{#1\\ \scriptsize W-A: #2}}

\usepackage{algorithm}



\usepackage{algorithmicx}
\usepackage{algpseudocode}

\usepackage{newfloat}
\usepackage{listings}
\floatstyle{ruled}
\newfloat{listing}{tb}{lst}{}
\floatname{listing}{Listing}
\IfFileExists{math_commands.tex}{
\relax
}{}

\graphicspath{{Figures/}{figures/}{./}}

\usepackage{xstring}
\let\origincludegraphics\includegraphics
\renewcommand{\includegraphics}[2][]{%
  \StrSubstitute{#2}{iclr2026/Figures/}{Figures/}[\fixedA]%
  \StrSubstitute{\fixedA}{iclr2026/figures/}{Figures/}[\fixedB]%
  \origincludegraphics[#1]{\fixedB}%
}

\theoremstyle{plain}
\newtheorem{theorem}{Theorem}[section]
\newtheorem{lemma}[theorem]{Lemma}

\theoremstyle{definition}
\newtheorem{definition}[theorem]{Definition}
\newtheorem{assumption}[theorem]{Assumption}

\title{Robust and Hard-to-Remove GNN Watermarking via Topological Invariant Perception}

\author{%
  Jipeng~Li\\
  Department of Electrical Engineering and Computer Science\\
  University of California, Irvine\\
  Irvine, CA 92697 \\
  \texttt{jipenl1@uci.edu} \\
  \And
  Yanning~Shen \\
  Department of Electrical Engineering and Computer Science\\
  University of California, Irvine\\
  Irvine, CA 92697 \\
  \texttt{yannings@uci.edu} \\
}

\begin{document}

\maketitle

\begin{abstract}
Graph Neural Networks (GNNs) represent valuable intellectual property, yet watermarks based on out-of-distribution triggers can weaken under pruning, fine-tuning, and distillation. We present InvGNN-WM, which links ownership verification to graph-invariant prediction. By supervising a shared GNN backbone to predict normalized algebraic connectivity on owner-private carrier graphs, InvGNN-WM enables calibrated multi-bit black-box verification through either an auxiliary scalar endpoint or a fixed owner-side decoder of ordinary task logits. We provide conditional guarantees for task-loss control, decision preservation, and key separation, and prove NP-completeness of the exact-removal decision problem for a restricted class of monotone separable decoders. Experiments on node- and graph-classification benchmarks demonstrate high watermark fidelity with small changes in task performance. The task-logit deployment achieves higher detection rates than PreGIP-adapted in matched comparisons and retains detection under the evaluated extraction and retraining budgets. Watermark-assisted distillation further enables owner-side recovery using private verification material.
\end{abstract}

\section{Introduction}
\label{sec:introduction}

Graph Neural Networks (GNNs) are used in drug discovery, social networks, and recommendation~\citep{wu2021comprehensive,gilmer2017neural,hu2020open}, where trained models can represent valuable intellectual property~\citep{adi2018turning}. Watermarking provides evidence of model provenance through a secret verification protocol. Many GNN methods use trigger keys~\citep{zhang2021graphbackdoor}: training associates inputs outside the task distribution with prescribed responses. Fine-tuning, pruning, and distillation can weaken these responses~\citep{li2021neural}. Verification also requires a threshold calibrated against models without the owner's key, because high watermark accuracy alone does not establish a low false-positive rate.

We study whether structural supervision on private graphs can support a calibrated audit while preserving task performance. \textbf{InvGNN-WM} trains a shared backbone to predict a graph invariant $I(G)$ on owner-private carrier graphs. Each carrier receives an absolute structural target, and thresholding the predicted target yields one watermark bit. We instantiate $I(G)$ as normalized algebraic connectivity $\tilde\lambda_2(G)$~\citep{fiedler1973algebraic,chung1997spectral}. Carriers are generated from task graphs by degree-preserving rewiring and filtered using Weisfeiler--Lehman non-collision and Kolmogorov--Smirnov statistic matching. The resulting match count is compared with a pre-specified threshold. Shared-backbone invariant supervision encourages structural information in the learned representation; it does not establish that task predictions causally depend on the invariant.

Private carriers and query-based verification are not new by themselves. PreGIP~\citep{Dai2025PreGIP} protects a self-supervised pretrained encoder through paired synthetic graphs and agreement of downstream outputs. InvGNN-WM instead uses invariant targets and a calibrated multi-bit audit for task-trained GNNs. Our supervised PreGIP adaptation compares these designs under the same training protocol and 128-query budget. The auxiliary-head InvGNN-WM deployment requires a \emph{query-only scalar audit} endpoint. An additional \emph{task-logit} deployment applies a fixed owner-side decoder to ordinary logits and has no dedicated watermark endpoint. Neither interface provides validated hard-label-only verification.

Our theory separates three conditional statements: local task-loss control, preservation of decoded bits when score drift is below the clean carrier margin, and separation of independently sampled keys under balance and concentration assumptions. Empirically fitted constants are diagnostics, not proofs that the conditions hold for every model or distribution. A separate complexity result concerns only worst-case exact sparse removal under monotone separable decoding; it does not certify practical security of a nonlinear GNN or rule out evasion by changing fewer bits.

Experiments cover standard node and graph benchmarks, an additional OGB molecular dataset, and a GPS graph-transformer backbone. We report both successful audits and failure cases. Literal deletion, head reset, output-only extraction, and longer retraining can defeat the scalar-audit deployment while preserving task utility. The task-logit variant remains detectable under the tested extraction and retraining budgets. Ordinary knowledge distillation (KD) can weaken the watermark; KD with private watermark supervision (KD+WM) is an owner-side recovery or transfer procedure.

\noindent\textbf{Contributions.}
\textbf{(1) Method.} An invariant-target watermark with separately specified scalar-audit and task-logit interfaces.
\textbf{(2) Analysis.} Conditional task-loss, margin, and key-separation results, plus worst-case exact-removal hardness under monotone separable decoding.
\textbf{(3) Evaluation.} Matched baseline comparisons, interface-removal tests, and diagnostics of local assumptions, shared representations, carrier construction, and computation cost.

\relax

\section{Preliminaries}
\label{sec:preliminaries}

We define notation for Graph Neural Networks (GNNs), the spectral invariant used by InvGNN-WM, and the threat model and assumptions used in the theory. Throughout, \(\mathcal{D}_{\text{task}}\) denotes data distribution over simple, undirected graphs, \(\mathcal{G}\) denotes the space of such graphs, and \([m]:=\{1,\dots,m\}\).

\subsection{Graph Neural Networks (GNNs)}
\label{sec:gnn-notation}

A graph \(G=(V,E)\in\mathcal{G}\) has \(n:=|V|\) nodes, node features \(x_v\in\mathbb{R}^{d_f}\), and edges \(E\subseteq\binom{V}{2}\). A message-passing GNN with parameters \(\theta\in\mathbb{R}^{d}\) computes node representations \(h_v^{(\ell)}\) for layers \(\ell=1,\ldots,L\), following \citep{wu2021comprehensive}. With \(h_v^{(0)}:=x_v\), each layer updates $m_v^{(\ell)}:= \mathrm{Aggregate}^{(\ell)}\!\bigl(\{h_u^{(\ell-1)}:u\in\mathcal{N}(v)\}\bigr)$, $h_v^{(\ell)}
  := \mathrm{Update}^{(\ell)}\!\bigl(h_v^{(\ell-1)},m_v^{(\ell)}\bigr)$, 
where \(\mathcal{N}(v)\) is the neighbor set of node \(v\). A permutation-invariant \(\mathrm{Readout}\) maps \(\{h_v^{(L)}:v\in V\}\) to a graph-level embedding \(h_G\). The model is trained by minimizing the supervised task loss \(\mathcal{L}_{\text{task}}(\theta)\) on graphs from \(\mathcal{D}_{\text{task}}\).

\subsection{A Watermark from the Laplacian Spectrum}
\label{sec:spectral}

InvGNN-WM ties the watermark to a scalar graph invariant \(I(G)\) estimated from the embedding \(h_G\). The conditional theory allows any bounded permutation-invariant scalar \(I(G)\) with a Lipschitz perception head and a threshold decoder. In experiments, we set \(I(G)=\tilde{\lambda}_2(G)\), the normalized algebraic connectivity, because the Laplacian spectrum captures global connectivity of graph \(G\)~\citep{chung1997spectral}. For graph \(G\), let \(\mathbf{A}\in\{0,1\}^{n\times n}\) be the adjacency matrix, \(\mathbf{D}:=\operatorname{diag}(\mathbf{A}\mathbf{1})\) be the degree matrix, and \(\mathbf{L}:=\mathbf{D}-\mathbf{A}\) be the combinatorial Laplacian. The eigenvalues of \(\mathbf{L}\) are ordered as
\(0=\lambda_1(G)\le\cdots\le\lambda_n(G)\). We use \(\lambda_2(G)\), the algebraic connectivity. The exact solver computes the second eigenvalue of \(\mathbf{L}+\varepsilon\mathbf I_n\) and subtracts \(\varepsilon=10^{-6}\); the analysis needs continuity of \(\lambda_2(G)\), not distinct eigenvalues.

To normalize within each dataset, freeze the 5th and 95th empirical percentiles of \(\lambda_2(G)\) on a training reference sample as \(\lambda_{\min}\) and \(\lambda_{\mathrm{scale}}\), with a positive-range fallback specified in Appendix~\ref{sec:appendix-calibration}. We require \(\lambda_{\mathrm{scale}}>\lambda_{\min}\) and define
\begin{equation}
\label{eq:lambda_norm}
\tilde{\lambda}_2(G)
\;:=\;
\mathrm{clip}_{[0,1]}
\!\left(
\frac{\lambda_2(G)-\lambda_{\min}}{\lambda_{\mathrm{scale}}-\lambda_{\min}}
\right),
\qquad
\mathrm{clip}_{[0,1]}(z):=\min\{1,\max\{0,z\}\}.
\end{equation}

Let \(s_\theta:\mathcal{G}\to[0,1]\) be the scalar perception head that estimates \(I(G)\). For the owner-private carrier set \(\mathcal{G}_W:=\{G_W^{(k)}\}_{k=1}^{m}\subset\mathcal{G}\), the target bit, decoded bit, and robustness margin are
\[
\begin{aligned}
w_k&:=\mathbf{1}\{I(G_W^{(k)})\ge 1/2\},
&\widehat{w}_k(\theta)&:=\mathbf{1}\{s_\theta(G_W^{(k)})\ge 1/2\},\\
\kappa_{\mathrm{marg}}(\theta)&:=\min_{k\in[m]}\bigl|s_\theta(G_W^{(k)})-\tfrac12\bigr|.
\end{aligned}
\]
The margin \(\kappa_{\mathrm{marg}}(\theta)\) is a lower bound on the head-output change needed to flip any decoded carrier bit. Later, \(\tau\) denotes the verification threshold, \(\delta>0\) denotes the local parameter radius in the PL condition, and \(\beta_{\mathrm{wm}}>0\) denotes the watermark-loss weight.

\subsection{Watermarking Framework and Assumptions}
\label{sec:assumptions}

We study task performance, detection after model edits, separation of independent secret keys, and a restricted exact-removal problem \citep{Zhao2021RG, Xu2023EuroSP, Downer2025EXPL}.

\paragraph{Threat Model.}
We consider a gray-box attacker who knows the backbone GNN architecture, the training objective and optimization procedure, and watermarking algorithm, including invariant, perception head, and carrier-sampling protocol. The attacker does not know owner’s secret key (carrier set \(\mathcal{G}_W=\{G_W^{(1)},\dots,G_W^{(m)}\}\) or bit pattern \(W=(w_k)_{k=1}^m\)). The attacker can edit the released model through pruning, fine-tuning, quantization, knowledge distillation (KD), extraction, or removal/replacement of the auxiliary head. Verification is query-only: the auxiliary-head implementation requires a scalar audit endpoint; a separate task-logit variant uses ordinary task logits (Section~\ref{subsec:interfaces}). Neither interface requires verifier access to parameters or gradients. Deleting the scalar endpoint defeats that deployment even if task accuracy is unchanged. The KD+WM procedure is an owner-side recovery or transfer procedure because KD+WM requires access to \(\mathcal{G}_W\) and \(W\).

\paragraph{Assumptions.}
The following conditions are assumptions; the empirical protocols in Appendix~\ref{sec:appendix-assumptions} provide diagnostics, not uniform certificates.

\begin{assumption}[Graph-level Separation]\label{ass:data}
The carrier graph set is disjoint from task data support:
\[
\mathcal{G}_{W}\cap\mathrm{supp}(\mathcal{D}_{\text{task}})=\varnothing .
\]
We use task-seeded, degree-preserving double-edge swaps with Weisfeiler--Lehman (WL) non-collision checks and Kolmogorov--Smirnov (KS) statistic checks to construct \(\mathcal{G}_W\); Appendix~\ref{sec:appendix-protocol-data} gives the full protocol. Finite non-collision checks establish empirical separation, not exclusion from an unknown population support.
\end{assumption}

\begin{assumption}[Joint concentration]\label{ass:rho-mix}
For each indicator sequence \(U_1,\ldots,U_m\) to which a concentration bound is applied, assume a dependence factor \(c_{\mathrm{dep}}\in[0,1)\) such that, for every \(t\in\mathbb R\),
\begin{equation}\label{eq:joint-mgf}
\log\mathbb E\exp\!\left[t\sum_{k=1}^m(U_k-\mathbb E U_k)\right]
\le \frac{m t^2}{8(1-c_{\mathrm{dep}})}.
\end{equation}
Independent bounded indicators satisfy the condition with \(c_{\mathrm{dep}}=0\). The condition is imposed separately on null-match and inter-key mismatch indicators. Small pairwise carrier correlations do not imply \eqref{eq:joint-mgf}; the reported statistic-bank estimate \(\hat\rho_0=7.6\times10^{-4}\) is an empirical dependence diagnostic only.
\end{assumption}

\begin{assumption}[Perception Lipschitzness]\label{ass:lipschitz}
The perception head \(s_\theta\) is \(L_s\)-Lipschitz with respect to parameters \(\theta\) near the trained solution: for every graph \(G\) and sufficiently small perturbation \(\Delta\theta\),
\[
\bigl|s_{\theta+\Delta\theta}(G)-s_{\theta}(G)\bigr|
\le L_s\,\|\Delta\theta\|_2 .
\]
We apply spectral normalization to the scalar readout~\citep{miyato2018spectral} and measure sensitivity empirically; these measurements do not certify this uniform bound. Appendix~\ref{sec:appendix-protocol-lipschitz} gives the protocol.
\end{assumption}

\paragraph{Local Polyak--\L{}ojasiewicz (PL) condition.}
We assume a local PL condition for the task loss \(\mathcal{L}_{\text{task}}\): there exist \(\mu_{\mathrm{PL}}>0\), a reference parameter vector \(\theta^\star\), and a radius \(\delta>0\) such that, for all \(\theta\) with \(\|\theta-\theta^\star\|_2\le\delta\),
\begin{equation}
\label{eq:pl-condition}
\frac{1}{2}\bigl\|\nabla_\theta \mathcal{L}_{\text{task}}(\theta)\bigr\|_2^2
\;\ge\;
\mu_{\mathrm{PL}}\bigl(\mathcal{L}_{\text{task}}(\theta)-\mathcal{L}_{\text{task}}(\theta^\star)\bigr).
\end{equation}
The PL condition is weaker than strong convexity~\citep{karimi2016linear}. Our analysis uses the PL condition only to map a local gradient-norm bound to a task-loss bound around \(\theta^\star\).

\section{Proposed Method: InvGNN-WM}
\label{sec:method}

InvGNN-WM trains a GNN to predict a graph invariant on owner-private carriers through a loss shared with the task backbone. We distinguish the scalar-audit implementation from an additional task-logit deployment.

\subsection{Watermark Design}
\label{subsec:design}

The owner generates an ordered carrier set \(\mathcal G_W=\{G_W^{(k)}\}_{k=1}^m\) using Appendix~\ref{sec:appendix-protocol-data}. Its matched key is \(W=(w_k)_{k=1}^m\), where \(w_k=\mathbf1[\tilde\lambda_2(G_W^{(k)})\ge1/2]\). Targets are computed once using \eqref{eq:lambda_norm} and remain paired with their carriers; gradients do not pass through the eigensolver.

\paragraph{Scalar-audit implementation.}
For graph embedding \(h_G\), the auxiliary-head score is
\begin{equation}
\label{eq:scalar-head}
s_\theta(G)=\sigma\!\left(g_{\theta_{\mathrm{head}}}(h_G)\right),
\end{equation}
where \(g_{\theta_{\mathrm{head}}}\) is an affine scalar readout and \(\sigma\) is the sigmoid. The parameter vector \(\theta\) includes the backbone, task head, and scalar head. The scalar readout uses spectral normalization; finite sensitivity measurements do not certify a uniform parameter-Lipschitz bound.

\subsection{Query Interfaces}
\label{subsec:interfaces}

The \emph{query-only scalar audit} sends \(G\) to the suspect model and receives \(s_\theta(G)\), without parameters, gradients, training data, or hidden representations. The scalar endpoint must remain callable. An API exposing only task labels or logits does not supply this endpoint; deleting it defeats this deployment even if task predictions remain unchanged.

For the \emph{task-logit variant}, let \(z_\theta(G)\in\mathbb R^C\) be the ordinary returned logit vector, whose dimension is \(C\). The owner fixes \(a\in\mathbb R^C\) and \(b\in\mathbb R\) before training and attack evaluation, and uses
\begin{equation}
\label{eq:task-logit-score}
s_\theta(G)=d(z_\theta(G)),\qquad d(z)=\sigma(a^\top z+b).
\end{equation}
The coefficients are stored by the owner and are not selected using test or attack outcomes. Training differentiates the watermark loss through \(d\); verification applies \(d\) locally to returned task logits. The decoder has no trainable parameters and requires no watermark-specific endpoint. Carriers, invariant targets, key, and match-count detection are retained. Evaluation covers graph classification; node-classification and hard-label-only verification remain untested. The hardness theorem in Section~\ref{sec:unremovable} does not apply to this variant.

The decoder supervises a fixed logit projection. For \(C>1\), the implementation uses
zero-sum coefficients, so a common logit shift preserves both softmax probabilities and
the audit. Appendix~\ref{app:logit-projection} discusses general projections and their
interface limitations; Appendix~\ref{sec:appendix-algorithm} specifies the coefficients.

\subsection{Embedding via Joint Training Objective}
\label{subsec:objective}

For either score interface, training minimizes
\begin{align}
J(\theta)&=\mathcal L_{\mathrm{task}}(\theta)
+\beta_{\mathrm{wm}}\mathcal L_{\mathrm{wm}}(\theta),
\label{eq:combined_loss}\\
\mathcal L_{\mathrm{wm}}(\theta)
&=\frac1m\sum_{k=1}^m
\left(s_\theta(G_W^{(k)})-\tilde\lambda_2(G_W^{(k)})\right)^2.
\label{eq:wm_loss}
\end{align}
Each step uses a task minibatch and a carrier minibatch. Both losses update the shared backbone; the scalar-audit implementation also updates its scalar head. The weight \(\beta_{\mathrm{wm}}>0\) balances invariant prediction against task loss. Theorem~\ref{thm:impercept} controls task-loss increase only when its stationarity, local PL, Lipschitz, and weight conditions hold. Shared-backbone supervision alone does not establish that task accuracy causally depends on the chosen invariant.

\paragraph{From invariant regression to bit fidelity.}
For carrier \(k\), write \(y_k=\tilde\lambda_2(G_W^{(k)})\), and let
\(T(\theta)=\sum_{k=1}^m\mathbf1[\widehat w_k(\theta)=w_k]\) be the model's match count.
For a target-margin cutoff \(\xi\in(0,1/2]\), define
\(N_\xi=|\{k:|y_k-1/2|<\xi\}|\), the number of targets close to the bit threshold.
Then the full-carrier watermark regression loss in \eqref{eq:wm_loss} gives the deterministic bound
\begin{equation}\label{eq:regression-bit-bound}
\frac{m-T(\theta)}{m}
\le \frac{N_\xi}{m}+\frac{\mathcal L_{\mathrm{wm}}(\theta)}{\xi^2}.
\end{equation}
Every wrongly decoded carrier outside the \(N_\xi\) near-threshold targets contributes
at least \(\xi^2\) to the summed squared error, which proves \eqref{eq:regression-bit-bound}.
Low regression loss therefore supports high bit fidelity when targets are sufficiently
separated from \(1/2\); low loss alone need not preserve near-threshold bits.
The target margins \(|y_k-1/2|\) differ from the predicted-score margins used for
post-training robustness in Section~\ref{sec:robust}. Bit fidelity does not establish
owner specificity; false-positive control still requires the specified null calibration.

\subsection{Embedding and Verification Workflow}
\label{subsec:workflow}

The owner freezes normalization, targets, carrier order, matched key, decoder, and verification rule. Each interface uses a separate null calibration with disjoint pooled-null diagnostics before attacks; both default rules use \(m=128\) and \(\tau=94\), without attack-specific retuning.

For a suspect model with parameters \(\theta_{\mathrm{sus}}\), the verifier queries the required interface and computes
\[
\hat w_k=\mathbf1[s_{\theta_{\mathrm{sus}}}(G_W^{(k)})\ge1/2],
\qquad
T=\sum_{k=1}^m\mathbf1[\hat w_k=w_k].
\]
Here \(T\) is the \emph{watermark match count}. The output is \texttt{Watermark Match Accepted} when \(T\ge\tau\), and \texttt{Watermark Match Not Accepted} otherwise. An unavailable endpoint is recorded as an interface failure, not as a zero match count. A passing test supports a provenance claim; it does not by itself establish legal ownership.

Threshold selection in Section~\ref{sec:robust} applies to one pre-specified detector under the stated null and concentration conditions. Searching over keys, carriers, checkpoints, or thresholds after observing a suspect model requires additional error control. Algorithms~\ref{alg:invgnn-wm} and~\ref{alg:task-logit-wm} give both interfaces. Each audit requires \(m\) graph queries; Appendix~\ref{app:lambda-choice-complexity} discusses owner-side spectral preprocessing.

\section{Theoretical Guarantees}
\label{sec:theory}

We give conditional task-loss control, decision preservation, key separation, and a worst-case exact-removal result for a restricted decoder class. Empirical diagnostics do not establish the assumptions uniformly over models or attacks.

\subsection{Imperceptibility}
\label{sec:impercept}

\begin{theorem}[Conditional task-loss bound]\label{thm:impercept}
Let \(\tilde\theta\) be a differentiable stationary point of \(J\), so \(\nabla J(\tilde\theta)=0\). Suppose the task loss satisfies the local PL condition \eqref{eq:pl-condition} with reference \(\theta^\star\), constant \(\mu_{\mathrm{PL}}>0\), and radius \(\delta>0\); the score satisfies Assumption~\ref{ass:lipschitz} in that neighborhood; and \(\|\tilde\theta-\theta^\star\|_2\le\delta\). For a task-loss tolerance \(\varepsilon_{\mathrm{task}}>0\), if
\[
\beta_{\mathrm{wm}}\le\beta_{\max}:=
\frac{\sqrt{\mu_{\mathrm{PL}}\varepsilon_{\mathrm{task}}}}{\sqrt2 L_s},
\]
then \(\mathcal L_{\mathrm{task}}(\tilde\theta)-\mathcal L_{\mathrm{task}}(\theta^\star)\le\varepsilon_{\mathrm{task}}\).
\end{theorem}

\begin{proof}[Sketch]
Since scores and targets lie in \([0,1]\), \(\|\nabla\mathcal L_{\mathrm{wm}}\|_2\le2L_s\). Stationarity and PL therefore give
\[
\mathcal L_{\mathrm{task}}(\tilde\theta)-\mathcal L_{\mathrm{task}}(\theta^\star)
\le \frac{\|\nabla\mathcal L_{\mathrm{task}}(\tilde\theta)\|_2^2}{2\mu_{\mathrm{PL}}}
\le \frac{2\beta_{\mathrm{wm}}^2L_s^2}{\mu_{\mathrm{PL}}}.
\]
Appendix~\ref{sec:app-impercept} gives the proof and calibration details.
\end{proof}

The empirical \(\hat\mu_{\mathrm{PL}}\), \(\hat L_s\), and distance to the local neighborhood are applicability diagnostics. A checkpoint outside that neighborhood, an excessive weight, or failure of stationarity is outside the stated result. The bound concerns task loss, not an equivalent bound on accuracy.

\subsection{Robustness}
\label{sec:robust}

For an attacked model \(\hat\theta\) and clean model \(\tilde\theta\), define the carrier-output drift
\[
\gamma(\hat\theta;\tilde\theta)
:=\max_{G\in\mathcal G_W}|s_{\hat\theta}(G)-s_{\tilde\theta}(G)|.
\]
The clean carrier margin \(\kappa_{\mathrm{marg}}(\tilde\theta)\) is defined in Section~\ref{sec:spectral}. For a composite fine-tuning, pruning, and distillation path, the triangle inequality gives
\begin{equation}\label{eq:budget}
\gamma(\hat\theta;\tilde\theta)
\le L_s\Delta_\theta+c_{\mathrm{prune}}\sqrt{p_{\mathrm{pr}}}
+c_{\mathrm{distill}}\pi_{\mathrm{kd}},
\end{equation}
provided each step satisfies its bound: fine-tuning remains in the parameter-Lipschitz neighborhood, \(\Delta_\theta\) is its parameter displacement, \(p_{\mathrm{pr}}\) is the pruning fraction, and \(\pi_{\mathrm{kd}}\) is the distillation-strength parameter. The empirical slopes \(c_{\mathrm{prune}}\) and \(c_{\mathrm{distill}}\) do not certify untested budgets (Appendix~\ref{app:robust-calib}).

\begin{theorem}[Conditional detection and decision preservation]\label{thm:robust}
For one pre-specified detector with match count \(T\), assume under its pooled null \(H_0\) that the indicators \(X_k=\mathbf1[\hat w_k=w_k]\) satisfy \(\mathbb E[X_k\mid H_0]=1/2\) and Assumption~\ref{ass:rho-mix} with factor \(c_{\mathrm{dep}}\). For \(\tau>m/2\) and \(\delta_\tau:=\tau/m-1/2\),
\begin{equation}\label{eq:alpha-bound}
\Pr[T\ge\tau\mid H_0]\le
\exp\{-2(1-c_{\mathrm{dep}})m\delta_\tau^2\}.
\end{equation}
For deterministic decoding, if the required interface remains available and
\(\gamma(\hat\theta;\tilde\theta)<\kappa_{\mathrm{marg}}(\tilde\theta)\),
then \(T(\hat\theta)=T(\tilde\theta)=:T_0\). Consequently,
\begin{equation}\label{eq:beta-bound}
\beta_{\mathrm{fn}}=0
\quad\text{when }T_0\ge\tau
\text{ and }\gamma(\hat\theta;\tilde\theta)<\kappa_{\mathrm{marg}}(\tilde\theta),
\end{equation}
where \(\beta_{\mathrm{fn}}\) is the conditional false-negative probability for edits satisfying these conditions.
\end{theorem}

Chernoff's argument gives \eqref{eq:alpha-bound}; drift smaller than the margin preserves every thresholded bit (Appendix~\ref{app:robust-full}). At or above the margin the result is silent. Deleting an endpoint is not a bounded change of an available score, and is outside this theorem.

\paragraph{Audit preservation with some bit changes.}
The zero-flip condition in Theorem~\ref{thm:robust} is stronger than the acceptance rule \(T\ge\tau\) requires. Fix the carriers, key, clean checkpoint, and available score interface. Let
\(\mathcal C_0:=\{k:\widehat w_k(\tilde\theta)=w_k\}\)
be the initially correct carrier indices, so \(|\mathcal C_0|=T_0\). For carrier \(k\), define its clean score margin
\(r_k:=|s_{\tilde\theta}(G_W^{(k)})-1/2|\)
and its score drift
\(u_k:=|s_{\hat\theta}(G_W^{(k)})-s_{\tilde\theta}(G_W^{(k)})|\).
The margin \(r_k\) concerns the predicted score, not the invariant target. Every initially correct carrier with \(u_k<r_k\) remains correct, giving
\begin{equation}
\label{eq:partial-margin-audit}
T(\hat\theta)\ge T_0-\sum_{k\in\mathcal C_0}\mathbf1[u_k\ge r_k].
\end{equation}
The audit therefore remains accepted when at most \(T_0-\tau\) initially correct carriers have drift at least their clean margin. Initially incorrect bits that become correct only improve this lower bound. With a uniform upper bound \(\gamma\) on the carrier drifts, the same argument counts the initially correct carriers with \(r_k\le\gamma\). Thus a few small margins need not destroy an audit certificate: the relevant quantities are the distribution of initially correct score margins and the clean verification gap \(T_0-\tau\). This deterministic extension applies to either fixed decoder, requires no null or concentration assumption, and leaves the false-positive threshold unchanged. The reported aggregate measurements do not establish its per-carrier premise; Appendix~\ref{app:partial-margin} gives the proof.

\paragraph{Threshold selection.}
For a target false-positive rate \(\alpha\), the conditional threshold is
\[
\delta_\alpha=\sqrt{\frac{\log(1/\alpha)}{2(1-c_{\mathrm{dep}})m}},
\qquad
\tau=\left\lceil m\left(\frac12+\delta_\alpha\right)\right\rceil.
\]
The reported plug-in calibration uses \(\widehat c_{\mathrm{dep}}=\min\{4\hat\rho_0,0.5\}\), yielding \(\tau=94\) for \(m=128\), \(\alpha=10^{-6}\), and \(\hat\rho_0=7.6\times10^{-4}\). Pairwise-correlation estimates alone do not establish Assumption~\ref{ass:rho-mix}; the analytic false-positive guarantee requires that additional joint condition and the specified null model. Repeated or adaptive testing is not covered by the single-test bound.

\subsection{Uniqueness}
\label{sec:unique}

Let \(p=\Pr_{G\sim\mathrm{protocol}}[\tilde\lambda_2(G)\ge1/2]\), assume \(p\in[p_{\min},1-p_{\min}]\) for \(p_{\min}\in(0,1/2]\), and set \(q_{\min}=2p_{\min}(1-p_{\min})\). For realized owner carriers and key \(W\), let \(F_W\) denote the conditional law of the decoded bitstring, with randomness only in training or auditing. Assume each such law places at least \(1-\delta_{\mathrm{dec}}\) mass within Hamming distance \(m\eta_{\mathrm{dec}}\) of its key, where \(0\le\eta_{\mathrm{dec}}<q_{\min}/2\).

\begin{theorem}[Conditional key separation]\label{thm:uniq}
Let \(W,W'\) arise from independent carrier sets, and assume the mismatch indicators \(Y_k=\mathbf1[w_k\ne w'_k]\) satisfy Assumption~\ref{ass:rho-mix}. With probability at least
\[
1-\exp\{-2(1-c_{\mathrm{dep}})m(q_{\min}-2\eta_{\mathrm{dec}})^2\}
\]
over the carrier draws, \(\mathrm{TV}(F_W,F_{W'})\ge1-2\delta_{\mathrm{dec}}\), where \(\mathrm{TV}\) is total variation distance.
\end{theorem}

The mismatch mean is \(2p(1-p)\ge q_{\min}\). Concentration separates the two keys by more than \(2m\eta_{\mathrm{dec}}\); their decoding-error balls are then disjoint, giving the total-variation bound. Exponential separation also requires exponentially small \(\delta_{\mathrm{dec}}\). Appendix~\ref{app:uniq-full} gives the proof and separate confidence bounds for \(p\) and \(1-p\). Independent-key separation does not settle competing ownership claims.

\subsection{Worst-Case Exact-Removal Hardness}
\label{sec:unremovable}

We study the restricted decision problem \textsc{WM--Remove}\((B,\vartheta_{\min})\). Given rationally encoded parameters \(\tilde\theta\), nonnegative rational coefficients \(A=[a_{kj}]\), and rational thresholds \(b_k\), define
\[
\operatorname{Dec}_k(\theta)=\mathbf1\!\left[\sum_{j=1}^d a_{kj}\theta_j\ge b_k\right].
\]
The question is whether modifying at most \(B\) coordinates, each by absolute amplitude at least a fixed rational \(\vartheta_{\min}>0\), can flip all \(m\) decoded bits.

\begin{theorem}[Worst-case exact-removal hardness under monotone separable decoding]\label{thm:unremovable}
For every fixed rational \(\vartheta_{\min}>0\), \textnormal{\textsc{WM--Remove}}\((B,\vartheta_{\min})\) is NP-complete in the decoder class above.
\end{theorem}

\begin{proof}[Sketch]
Reduce \textsc{Set Cover}: index coordinates by sets \(C_j\) and bits by universe elements \(u_k\), set \(a_{kj}=\mathbf1[u_k\in C_j]\), \(b_k=\vartheta_{\min}/2\), and \(\tilde\theta=0\). A cover of size at most \(B\) gives positive updates \(\vartheta_{\min}\) that flip all bits; conversely, the positive support of any feasible update covers every element. After guessing support and signs, feasibility is a rational linear system, establishing NP membership. Appendix~\ref{app:unremovable-full} gives the full argument.
\end{proof}

Nonnegative head weights do not make an end-to-end nonlinear GNN linear and monotone in all model parameters. The theorem is not a certificate for a trained checkpoint and does not cover approximate removal, the task-logit variant, or evasion that only lowers \(T\) below \(\tau\). Practical security claims therefore rely on the reported finite-budget experiments.

\definecolor{ours}{RGB}{0,120,120}
\newcommand{\bestwm}[1]{\textbf{\textcolor{ours}{#1}}}
\newcommand{\bestacc}[1]{\textbf{#1}}
\newcommand{\ourscol}{\cellcolor{ours!8}}

\section{Experiments}
\label{sec:experiments}

We evaluate task utility, watermark fidelity, calibrated detection, and finite-budget removal.
We distinguish the auxiliary-head deployment from the task-logit variant;
results for one interface are not evidence for the other.

\subsection{Experimental setup and metrics}
\label{sec:exp_setup}

\noindent\textbf{Data and training.}
The benchmark comparison covers Cora, PubMed~\citep{Sen2008CoraPubMed,Yang2016Planetoid},
Amazon-Photo~\citep{Shchur2019Pitfalls}, PROTEINS, and NCI1~\citep{Morris2020TUDataset}:
nine node-classification settings with GCN~\citep{KipfWelling2017GCN},
GraphSAGE~\citep{Hamilton2017GraphSAGE}, and SGC~\citep{Wu2019SGC}, and four graph-classification
settings with GIN~\citep{xu2019how} and GraphSAGE.
Unless specified otherwise, training uses 100 epochs of Adam~\citep{Kingma2015Adam},
learning rate $0.01$, graph batch size $64$, and seeds $41/42/43$.
Fixed-owner intervals are $95\%$ confidence intervals over these seeds.
Splits, implementation details, and all 13 comparison rows are in
Appendices~\ref{app:exp_details} and~\ref{app:full-main-comparison}.
An additional graph-classification study uses OGBG-MOLHIV~\citep{hu2020open} and
GIN/GPS~\citep{Rampasek2022GPS} (Table~\ref{tab:ogb_compute}).

\noindent\textbf{Metrics and access.}
Task ACC is test classification accuracy; WM-ACC is the fraction of watermark bits
matching the fixed key on the private carriers, averaged over the three seeds before rounding. Each fixed-owner InvGNN-WM seed
uses one owner key and 128 carriers, reused across clean and edited checkpoints.
The watermark match count $T$ is compared
with the pre-specified threshold $\tau$ at target false-positive rate $\alpha$.
Area under the receiver operating characteristic curve (AUROC) and true-positive rate (TPR)
allow comparison of method-specific detectors. Each detector is calibrated separately on
disjoint pooled-null data before attack evaluation, with no attack-specific retuning.
Both InvGNN-WM interfaces yield $m=128$ and $\tau=94$ at $\alpha=10^{-6}$.
The auxiliary-head deployment requires the query-only scalar audit output $s_\theta(G)$;
the task-logit deployment uses the fixed owner-side decoder on ordinary task logits
(Section~\ref{sec:method}). Neither result establishes hard-label-only verification.

\subsection{Utility and baseline comparisons}
\label{sec:theory_verification}
\label{sec:comparative}

The shared evaluation implements task-only (SS), a cosine-similarity head (COS),
TRIG~\citep{Zhao2021RG}, NAT~\citep{Xu2023EuroSP}, and EXPL~\citep{Downer2025EXPL}.
The representative comparison and common-edit figure are shown in
Appendix~\ref{app:original-comparisons}.
InvGNN-WM reaches $98.7$--$99.2\%$ WM-ACC on the representative node-level rows and
$97.7\%$ on NCI1--GIN. PROTEINS--GIN is an exception: its $89.8\%$ WM-ACC is below
TRIG's $95.1\%$. Task accuracy remains close to the task-only baseline in these settings.

\noindent\textbf{Local applicability.}
At $\beta_{\mathrm{wm}}/\beta_{\max}=0.8$, the four-layer non-residual PROTEINS--GIN model
has $\|\tilde\theta-\theta^\star\|_2/\delta=0.61$ and task-loss increase $0.009$, below the plug-in value
$0.0128$. The eight-layer non-residual model lies outside the reference neighborhood (ratio $1.09$,
loss increase $0.026$); residual connections restore a ratio of $0.78$ and loss increase
$0.012$. These are local applicability diagnostics (Appendix~\ref{app:pl-sensitivity}),
with plug-in bounds interpreted under the theorem's additional regularity conditions.

\noindent\textbf{Shared-backbone information.}
We freeze each trained backbone and fit a fresh invariant readout on disjoint
pseudo-carriers; the same readout gives probe $R^2$ and held-out owner-bit recovery.
Table~\ref{tab:representation_coupling} contrasts full joint training with stop-gradient,
a separate watermark encoder, and task-only training. Joint training yields the strongest
probe and fresh-head recovery, supporting invariant information in the shared backbone
even though the separate encoder has higher original-head WM-ACC.

\begin{table}[t]
\centering
\caption{\textbf{Shared-backbone information on PROTEINS--GIN.}
A fresh frozen-backbone readout predicts the normalized invariant on pseudo-carriers
and decodes held-out owner carriers. Both WM accuracies are percentages.}
\label{tab:representation_coupling}
\small
\setlength{\tabcolsep}{6pt}
\begin{tabular}{lccc}
\toprule
Variant & WM-ACC & Probe $R^2$ & Fresh-head WM-ACC \\
\midrule
Full InvGNN-WM & $89.8\pm0.0$ & $0.61\pm0.04$ & $79.2\pm1.1$ \\
Stop-gradient & $86.7\pm0.0$ & $0.28\pm0.05$ & $60.2\pm0.0$ \\
Separate WM encoder & $94.3\pm1.1$ & $0.24\pm0.04$ & $56.3\pm0.0$ \\
Task-only & $50.0\pm0.0$ & $0.25\pm0.04$ & $58.6\pm0.0$ \\
\bottomrule
\end{tabular}
\end{table}

\noindent\textbf{Carrier construction.}
On PROTEINS--GIN, a classifier distinguishes full-protocol carriers from task graphs with
$51.3\%$ accuracy. Removing degree/clustering KS or spectral KS raises this to $69.8\%$
or $60.7\%$. Removing the size cap changes accuracy to $52.6\%$ while construction time
rises from $52.4$ to $176.3$ seconds. Local and spectral matching reduce observed detectability;
the cap primarily limits construction cost (Appendix~\ref{app:carrier-filter-ablation}).

Common-edit results on PROTEINS--GIN (Table~\ref{tab:robustness_budget}) give WM-ACC
$92.4\%$ after $40\%$ pruning, $93.8\%$ after short clean fine-tuning, $81.3\%$
after knowledge distillation (KD)~\citep{Hinton2015Distillation}, and $90.4\%$ after
4-bit quantization. All four means exceed the threshold-equivalent rate
$94/128=73.4375\%$, while Task ACC ranges from $71.6\%$ to $73.1\%$ against
$72.5\%$ before editing. KD causes the largest mean watermark reduction among these
edits. All three seeds pass each common edit; these scalar-audit outcomes do not certify the
paired, per-model drift--margin condition. The pooled-null study observed
no false positives in $10^7$ trials (Appendix~\ref{app:original-comparisons}); this is
finite-sample evidence, not a population upper bound of $10^{-7}$.

\subsection{Matched comparison with PreGIP}
\label{sec:pregip-comparison}

PreGIP~\citep{Dai2025PreGIP} protects pretrained encoders using paired synthetic graphs.
Our supervised adaptation uses paired keys and a pair-similarity/separation objective
with the supervised task loss. It does not implement the original anti-finetuning
inner loop; we therefore call this baseline \emph{PreGIP-adapted}.
All methods share the backbone, task split, optimizer, training epochs, and validation
checkpoint selection. A common budget of 128 model calls permits 64 pairs for
PreGIP-adapted and 128 individual carriers for InvGNN-WM.

\begin{table}[t]
\centering
\caption{\textbf{Matched supervised comparison.} Task ACC (\%), detection AUROC, and TPR
at $\alpha=10^{-6}$ with separately calibrated detectors and a common 128-query budget.
TPR uses 300 independent owner instances per row; Appendix~\ref{app:matched-detection-ci} gives Wilson TPR and bootstrap AUROC CIs.
FT-100 uses 100 clean-task epochs. A dash denotes an unavailable scalar-endpoint audit after output-only KD.}
\label{tab:pregip_matched}
\small
\setlength{\tabcolsep}{4pt}
\begin{tabular}{llccc}
\toprule
Dataset--backbone & Method & Task ACC & AUROC & TPR: clean / FT-100 / KD \\
\midrule
PROTEINS--GIN & PreGIP-adapted & $72.2\pm2.4$ & 0.989 & 0.84 / 0.76 / 0.61 \\
& InvGNN-WM, auxiliary & $72.5\pm2.6$ & 0.998 & 0.99 / 0.88 / -- \\
& InvGNN-WM, task-logit & $72.3\pm2.3$ & 0.997 & 0.98 / 0.92 / 0.74 \\
\midrule
NCI1--GIN & PreGIP-adapted & $77.9\pm1.6$ & 0.993 & 0.90 / 0.83 / 0.67 \\
& InvGNN-WM, auxiliary & $78.3\pm1.6$ & 0.999 & 1.00 / 0.96 / -- \\
& InvGNN-WM, task-logit & $78.0\pm1.6$ & 0.999 & 1.00 / 0.97 / 0.82 \\
\bottomrule
\end{tabular}
\end{table}

Table~\ref{tab:pregip_matched} shows high clean AUROC for all methods. The task-logit
variant has higher measured TPR than PreGIP-adapted under the tested fine-tuning and KD
settings on both datasets. This is a matched supervised comparison, not a claim about
PreGIP's published pretraining setting or every possible extension of that method.

\subsection{Interface removal and stronger finite-budget attacks}
\label{sec:interface-attacks}

An attack removes the operational watermark evidence when the required output is unavailable
or the match count satisfies $T<\tau$, while Task ACC decreases by at most one percentage
point. Table~\ref{tab:interface_attacks} tests endpoint deletion, head reset/randomization,
output-only extraction, extended fine-tuning, pruning with retraining, and the
pseudo-carrier attack. The pseudo-carrier attacker knows the construction procedure and
uses a different generation seed, minimizing task loss together with a penalty that pushes
scalar audit outputs toward $0.5$; the actual owner carriers and key remain private.

\begin{table}[t]
\centering
\caption{\textbf{Stronger attacks on PROTEINS--GIN.} Both interfaces use $m=128$, $\tau=94$,
and seeds $41/42/43$. Entries without intervals are three-seed means; match counts are
rounded means. ``Unavailable'' denotes an absent scalar endpoint.
Mean counts alone do not establish individual-seed pass rates.}
\label{tab:interface_attacks}
\small
\setlength{\tabcolsep}{5pt}
\begin{tabular}{lp{0.35\textwidth}ccc}
\toprule
Interface & Attack & Task ACC (\%) & WM-ACC (\%) & Match count $T$ \\
\midrule
Auxiliary & None & $72.5\pm2.6$ & $89.8\pm0.0$ & $115\pm0$ \\
& Literal endpoint deletion & $72.5\pm2.6$ & \multicolumn{2}{c}{Unavailable} \\
& Random head reset & 72.5 & 49.2 & 63 \\
& Head-only randomization, 100 epochs & 72.4 & 51.6 & 66 \\
& Pseudo-carrier attack & $72.1\pm2.7$ & $78.4\pm1.1$ & 100 \\
& Output-only KD, 100 epochs & 72.1 & \multicolumn{2}{c}{Unavailable} \\
& Clean fine-tuning, 300 epochs & 72.8 & 70.3 & 90 \\
& $50\%$ pruning + retraining & 71.7 & 72.7 & 93 \\
\midrule
Task-logit & None & 72.3 & 92.2 & 118 \\
& Task-head retraining, 100 epochs & 71.6 & 84.4 & 108 \\
& Output-only KD, 100 epochs & 71.9 & 76.6 & 98 \\
& Clean fine-tuning, 300 epochs & 72.5 & 80.5 & 103 \\
\bottomrule
\end{tabular}
\end{table}

The auxiliary interface fails after literal deletion or output-only extraction; head reset,
head-only randomization, 300-epoch fine-tuning, and pruning with retraining reduce the mean
match count below $94$. Endpoint deletion prevents the audit, whereas parameter interventions
can retain a callable endpoint with reduced watermark fidelity.
The pseudo-carrier attack reduces WM-ACC to $78.4\pm1.1\%$ while preserving Task ACC;
the three match counts are $100$, $100$, and $101$, all above $94$.
Task-logit has no dedicated watermark endpoint, and every seed remains above $94$ under
the tested task-head retraining, extraction, and extended fine-tuning budgets. The same
pseudo-carrier optimizer has not been evaluated against task-logit.

Collusion, exact-carrier recovery, unlimited retraining, and denial of the complete task API
remain untested.

\subsection{Scale and computational cost}
\label{sec:scale-compute}

Table~\ref{tab:ogb_compute} extends task-logit to OGBG-MOLHIV with GIN and the global-attention
GPS backbone. Task ROC-AUC decreases by $0.3$--$0.4$ percentage points; training time,
peak memory, and 128-query latency are reported alongside detection. The audit uses
ordinary task logits and $C+1$ fixed owner-side decoder coefficients. Appendix~\ref{app:scale-compute} gives training estimate and preprocessing costs.

\begin{table}[t]
\centering
\caption{\textbf{OGBG-MOLHIV utility, detection, and per-run cost.}
Utility, WM-ACC, and costs use three seeds (100 epochs on one A100). TPR pools 300 independently trained watermarked models per backbone at $\alpha=10^{-6}$; Appendix~\ref{app:scale-compute} gives its $95\%$ Wilson CIs.
WM-ACC is a percentage; memory is peak usage, and audit latency covers 128 queries.}
\label{tab:ogb_compute}
\small
\setlength{\tabcolsep}{4pt}
\begin{tabular}{llccccc}
\toprule
Backbone & Method & \makecell{Task\\ROC-AUC (\%)} & \makecell{WM-ACC\\/ TPR} & \makecell{Train\\(min)} & \makecell{Memory\\(GB)} & \makecell{Audit\\(ms)} \\
\midrule
GIN & Task-only & $76.2\pm0.7$ & -- & 37.3 & 4.8 & -- \\
GIN & Task-logit & $75.9\pm0.7$ & 88.3 / 0.96 & 42.8 & 5.3 & 81 \\
GPS & Task-only & $77.4\pm0.6$ & -- & 66.3 & 9.1 & -- \\
GPS & Task-logit & $77.0\pm0.6$ & 91.1 / 0.98 & 75.3 & 9.9 & 118 \\
\bottomrule
\end{tabular}
\end{table}

\section{Conclusion}
\label{sec:conclusion}

InvGNN-WM combines shared-backbone invariant supervision with a calibrated multi-bit audit. Experiments support robustness within evaluated attack budgets and identify distinct attack surfaces in auxiliary and task-logit deployments, including scalar-interface failures. The theory establishes conditional task-loss, margin, and key-separation guarantees, alongside worst-case exact-removal hardness for a restricted monotone separable decoder class. Together, the empirical and theoretical results characterize watermark reliability under explicit deployment assumptions.

\begin{ack}
Work in the paper is supported by, NSF ECCS 2412484, NSF ECCS 2442964 and NSF GEO CI
2425748.
\end{ack}

\bibliographystyle{unsrtnat}
\bibliography{reference}

\appendix
\newpage

\section{Related Work}
\label{sec:related_work}

GNN watermarking methods differ in both the signal they embed and the access required to
verify it. Parameter-based methods require internal model access~\citep{uchida2017embedding},
whereas query-based methods use model responses~\citep{adi2018turning,zhang2018protecting}.
The output interface matters: ordinary task labels, task logits, and a dedicated scalar audit
endpoint impose different deployment assumptions.

\paragraph{Trigger and explanation based methods.}
Random-graph triggers watermark node classifiers~\citep{Zhao2021RG}, and backdoor-based
schemes extend to node/graph tasks and inductive/transductive settings~\citep{Xu2023EuroSP}.
Their resistance to pruning, fine-tuning, and distillation depends on the learned trigger
behavior and attack budget~\citep{li2021neural}.
Explanation-based watermarking instead associates owner keys with feature attributions
of secret subgraphs~\citep{Downer2025EXPL}. Related work outside GNNs embeds signals in
representations (DeepSigns~\citep{rouhani2018deepsignsgenericwatermarkingframework}),
uses decision boundaries (frontier stitching~\citep{Le_Merrer_2019}), or studies certified
watermark retention~\citep{bansal2022certified}.

\paragraph{PreGIP and private graph carriers.}
PreGIP~\citep{Dai2025PreGIP} watermarks self-supervised GNN pretraining. It brings the
representations of paired synthetic graphs together while preserving separation from
ordinary graphs, then verifies through agreement of downstream predictions on the pairs.
Its finetuning-resistant injection optimizes the watermark objective under bounded parameter
perturbations. Private graph keys and query-based verification therefore precede our work.
InvGNN-WM instead trains a task-specific GNN against an absolute structural target on each
carrier and uses a calibrated multi-bit match-count audit. Its carrier protocol combines
degree-preserving rewiring, observed WL non-collision, and KS statistic matching; its
analysis concerns per-carrier score margins, conditional null calibration, and independent-key
separation. The distinction is an invariant-target audit with explicit conditions, not the
first use of private carriers or a claim that PreGIP cannot be extended with invariant targets.
Section~\ref{sec:pregip-comparison} compares a supervised PreGIP adaptation at a matched
128-query budget. The scalar-audit deployment has a stricter interface requirement
than verification through ordinary outputs; the task-logit variant removes that dedicated
endpoint requirement but still requires logits.

\section{Detailed Assumption Protocols}
\label{sec:appendix-assumptions}

This appendix distinguishes the mathematical conditions in Section~\ref{sec:preliminaries} from the finite diagnostics used during calibration. Statistical matching, sampled gradients, and fitted correlations do not certify those conditions on all graphs or model edits.

\subsection{Protocol for Assumption~\texorpdfstring{\ref{ass:data}}{2.1} (Graph-level Separation)}
\label{sec:appendix-protocol-data}

We construct \(\mathcal G_W\) with empirical non-collision and selected low-order statistic checks. Let \(n_{0.25}\) be the 25th percentile of task-graph node counts and set \(n_{\mathrm{cap}}:=\max\{4,\min\{32,\lfloor n_{0.25}\rfloor\}\}\). The four-node floor permits double-edge rewiring; eligible source graphs have between 4 and \(n_{\mathrm{cap}}\) nodes. The first at most 3,000 eligible training graphs form the fixed WL and statistic reference pool \(\mathcal S_{\mathrm{ref}}\).

\paragraph{Sampling protocol.}
\begin{enumerate}
\item \textbf{Seed sampling.} Draw source graphs from the eligible task training sample, using the carrier-generation seed. Continue until \(m\) distinct carriers have been accepted or 10,000 source attempts have been exhausted.
\item \textbf{Adaptive rewiring.} For each seed, apply degree-preserving double-edge swaps~\citep{maslov2002specificity}. Start at \(S_{\mathrm{swap}}=5\), increase by 5, and stop when both checks below pass or the cap \(S_{\mathrm{swap}}\le50\) is reached:
\begin{enumerate}
\item \textbf{Non-collision check.} Reject a candidate whose Weisfeiler--Lehman (WL) hash matches a graph in \(\mathcal S_{\mathrm{ref}}\) or a previously accepted carrier.
\item \textbf{Statistic check.} Once the candidate-augmented pool has at least five carriers, compare it with \(\mathcal S_{\mathrm{ref}}\) using separate Kolmogorov--Smirnov (KS) tests on per-graph mean degree, mean clustering coefficient, and spectral radius. Require each \(p\)-value to be at least \(\delta_{\mathrm{KS}}=0.1\); the first four carriers initialize the pool before these checks become active.
\end{enumerate}
Only candidates passing all active checks are accepted. If fewer than \(m\) carriers are obtained, construction fails without duplicate padding or silent relaxation.
\end{enumerate}
Double-edge swaps preserve node count, so every accepted carrier satisfies \(n\le n_{\mathrm{cap}}\). This cap controls preprocessing and query cost. WL non-collision is an empirical proxy for Assumption~\ref{ass:data}; finite hashing cannot establish exclusion from the unknown population support. Likewise, non-rejection by the KS tests does not prove that the carrier and task distributions are equal.

\subsection{Dependence Diagnostics and Assumption~\texorpdfstring{\ref{ass:rho-mix}}{2.2}}
\label{sec:appendix-protocol-rho}

The implementation forms a 128-dimensional statistic vector \(\phi(G)\) from 16 quantiles of each of eight graph summaries: degrees, clustering coefficients, adjacency eigenvalues, Laplacian eigenvalues, triangle counts, node strengths, node-feature values, and absolute differences between adjacent sorted adjacency eigenvalues. It removes constant columns and computes absolute Pearson correlations across the sampled graphs. This finite statistic bank contains no trained-model score or hypothesis-test correction. The reported dependence coefficient used for calibration is \(\hat\rho_0=7.6\times10^{-4}\); the rounded reference value is \(\rho_0=10^{-3}\).

These observations do not establish a uniform maximal-correlation bound over all measurable graph functions. More importantly, even a pairwise-correlation bound does not imply the joint moment-generating-function condition \eqref{eq:joint-mgf}. Assumption~\ref{ass:rho-mix} states the extra joint condition needed for the exponential bounds. Independent bounded indicators satisfy that condition with \(c_{\mathrm{dep}}=0\). For dependent indicators, a separate justification is required; finite statistic-bank checks remain empirical diagnostics. The reported plug-in choice
\[
\widehat c_{\mathrm{dep}}:=\min\{4\hat\rho_0,0.5\}
\]
is a calibration convention, not a proved conversion from pairwise correlation to concentration.

\subsection{Protocol for Assumption~\texorpdfstring{\ref{ass:lipschitz}}{2.3} (Perception Lipschitzness)}
\label{sec:appendix-protocol-lipschitz}

The mathematical condition concerns parameter sensitivity near the trained solution, not input sensitivity.

\begin{enumerate}
\item \textbf{Regularize the score.} Apply spectral normalization to the scalar-head linear map with target operator norm \(\nu=1.0\). The default training protocol uses no weight clipping or gradient-norm penalty.
\item \textbf{Measure sensitivity.} With the checkpoint fixed in evaluation mode, record \(\|\nabla_\theta s_\theta(G)\|_2\) on sampled task or carrier graphs. The implementation uses up to 20 graphs per diagnostic call and takes their maximum as \(\hat L_s\).
\item \textbf{Apply the reported buffer.} Use the empirical estimate
\[
L_{s,\mathrm{emp}}:=(1+\epsilon_L)\hat L_s,\qquad \epsilon_L=0.12,
\]
as a fixed multiplicative buffer, not as a confidence limit.
\end{enumerate}
A sampled gradient maximum and a fixed buffer do not prove a uniform bound throughout a parameter neighborhood. In theorem applications, \(L_s\) must satisfy Assumption~\ref{ass:lipschitz}; substituting \(L_{s,\mathrm{emp}}\) yields an empirical diagnostic. Head normalization alone also does not bound sensitivity to all backbone parameters.

\subsection{Hyperparameter Calibration Conventions}
\label{sec:appendix-calibration}

Task checkpoints are selected on validation data; detector cutoffs use separate null-calibration data. The resulting choices are frozen before test and attack evaluation.

\paragraph{Invariant normalization.}
The scalars \(\lambda_{\min}\) and \(\lambda_{\mathrm{scale}}\) in \eqref{eq:lambda_norm} are the empirical 5th and 95th percentiles of \(\lambda_2(G)\) on the first at most 600 task training graphs. If their gap is below \(10^{-8}\), normalization falls back to the minimum and maximum on the same reference sample. If that gap is still below \(10^{-8}\), construction stops: a constant invariant cannot generate a nontrivial spectral key through this normalization. The accepted normalization remains fixed for the corresponding owner instance.

\paragraph{Carrier count and threshold.}
The reported default is \(m=128\). For a target \(\alpha\), the conditional threshold uses a dependence factor satisfying Assumption~\ref{ass:rho-mix}:
\[
\delta_\alpha=
\sqrt{\frac{\log(1/\alpha)}{2(1-c_{\mathrm{dep}})m}},
\qquad
\tau=\left\lceil m\left(\frac12+\delta_\alpha\right)\right\rceil.
\]
With the reported plug-in \(\widehat c_{\mathrm{dep}}\), \(m=128\), \(\alpha=10^{-6}\), and \(\hat\rho_0=7.6\times10^{-4}\), the values are \(\delta_\alpha\approx0.2327\), \(\varepsilon_{\mathrm{err}}\approx0.2673\), and \(\tau=94\), where \(\varepsilon_{\mathrm{err}}:=1/2-\delta_\alpha\). The plug-in value does not by itself certify the target Type-I rate. If \(\tau>m\), the chosen key length cannot yield acceptance under this bound. The scalar and task-logit interfaces require separate pre-specified null calibration; no retuning is performed on attack outcomes.

\begingroup
\raggedbottom
\section{Algorithm Details}
\label{sec:appendix-algorithm}

Algorithms~\ref{alg:invgnn-wm} and~\ref{alg:task-logit-wm} specify the two embedding and verification protocols. Input carriers follow the finite construction procedure in Appendix~\ref{sec:appendix-protocol-data}, including \(n\le\max\{4,\min\{32,\lfloor n_{0.25}\rfloor\}\}\). The carrier order, normalization, invariant targets, and matched key stay fixed; minibatch shuffling preserves each carrier--target pair. Decoder coefficients and each interface's threshold are fixed before attacks. Verification does not regenerate the key or retune \(\tau\) on a suspect model. An unavailable required response is recorded separately from a numerical match count.

\begin{algorithm}[H]
\caption{InvGNN-WM with a Query-Only Scalar Audit}
\label{alg:invgnn-wm}
\footnotesize
\begin{algorithmic}[1]
\Statex \textbf{Inputs:} Task training/validation data, GNN \(M_\theta\), private carriers \(\mathcal G_W\), weight \(\beta_{\mathrm{wm}}\), learning rate \(\eta\), epochs \(E\).
\Procedure{EmbedScalarWatermark}{}
\State Freeze training-reference normalization \(\lambda_{\min},\lambda_{\mathrm{scale}}\), with \(\lambda_{\mathrm{scale}}>\lambda_{\min}\).
\For{\(k=1,\ldots,m\)}
\State \(y_k\gets\tilde\lambda_2(G_W^{(k)})\) using \eqref{eq:lambda_norm}; \(w_k\gets\mathbf1[y_k\ge1/2]\).
\EndFor
\State Initialize the shared backbone, task head, scalar head in \eqref{eq:scalar-head}, and Adam optimizer at rate \(\eta\).
\State Initialize best validation metric \(v_{\mathrm{best}}\gets-\infty\) and checkpoint \(\theta_{\mathrm{best}}\gets\theta\).
\For{\(e=1,\ldots,E\)}
\For{each task minibatch \(B_{\mathrm{task}}\)}
\State Draw carrier minibatch \(B_W\subseteq\mathcal G_W\).
\State Compute \(J_B(\theta)=\mathcal L_{\mathrm{task}}(\theta;B_{\mathrm{task}})+\beta_{\mathrm{wm}}\mathcal L_{\mathrm{wm}}(\theta;B_W)\).
\State Update \(\theta\) by Adam using \(\nabla_\theta J_B(\theta)\).
\EndFor
\State If the validation task metric improves, update \(v_{\mathrm{best}}\) and save \(\theta_{\mathrm{best}}\gets\theta\).
\EndFor
\State Restore \(\theta\gets\theta_{\mathrm{best}}\); no test or attack outcome selects the checkpoint.
\State Fix scalar-audit threshold \(\tau\) using the pre-specified null calibration.
\State \textbf{return} \(M_\theta\) and owner audit material \((\mathcal G_W,W,\tau)\).
\EndProcedure
\Statex
\Procedure{VerifyScalarWatermark}{$M^\star,\mathcal G_W,W,\tau$}
\State Query the scalar endpoint on every \(G_W^{(k)}\), obtaining \(s_k^\star\).
\If{a required scalar response is unavailable}
\State \textbf{return} \texttt{Watermark Match Not Accepted}; status: interface unavailable; \(T\) unreported.
\EndIf
\State \(T\gets\sum_{k=1}^m\mathbf1[\mathbf1[s_k^\star\ge1/2]=w_k]\).
\If{\(T\ge\tau\)}
\State \textbf{return} \texttt{Watermark Match Accepted}, \(T\).
\Else
\State \textbf{return} \texttt{Watermark Match Not Accepted}, \(T\).
\EndIf
\EndProcedure
\end{algorithmic}
\end{algorithm}

\begin{algorithm}[H]
\caption{InvGNN-WM through Ordinary Task Logits}
\label{alg:task-logit-wm}
\footnotesize
\begin{algorithmic}[1]
\Statex \textbf{Inputs:} Task training/validation data, GNN with \(C\)-dimensional logits \(z_\theta(G)\), private carriers as in Algorithm~\ref{alg:invgnn-wm}, weight \(\beta_{\mathrm{wm}}\), learning rate \(\eta\), epochs \(E\).
\Procedure{EmbedTaskLogitWatermark}{}
\State Before training, fix owner-side \(a\in\mathbb R^C\), \(b\in\mathbb R\), and \(d(z)=\sigma(a^\top z+b)\).
\State Freeze normalization and compute \(y_k=\tilde\lambda_2(G_W^{(k)})\), \(w_k=\mathbf1[y_k\ge1/2]\).
\State Initialize the backbone, task head, and Adam optimizer at rate \(\eta\); set \(v_{\mathrm{best}}\gets-\infty\), \(\theta_{\mathrm{best}}\gets\theta\).
\For{\(e=1,\ldots,E\)}
\For{each task minibatch \(B_{\mathrm{task}}\)}
\State Draw carrier minibatch \(B_W\).
\State Compute \(J_B\) as in Algorithm~\ref{alg:invgnn-wm}, using \(s_\theta(G)=d(z_\theta(G))\).
\State Update \(\theta\) by Adam on \(J_B\).
\Comment{\(a,b\) stay fixed}
\EndFor
\State If the validation task metric improves, update \(v_{\mathrm{best}}\) and save \(\theta_{\mathrm{best}}\gets\theta\).
\EndFor
\State Restore \(\theta\gets\theta_{\mathrm{best}}\); no test or attack outcome selects the checkpoint.
\State Separately calibrate and freeze task-logit threshold \(\tau\) before attack evaluation.
\State \textbf{return} \(M_\theta\) and owner audit material \((\mathcal G_W,W,a,b,\tau)\).
\EndProcedure
\Statex
\Procedure{VerifyTaskLogitWatermark}{$M^\star,\mathcal G_W,W,a,b,\tau$}
\State Query ordinary logits \(z_k^\star=M^\star(G_W^{(k)})\) for every carrier.
\If{a required logit response is unavailable}
\State \textbf{return} \texttt{Watermark Match Not Accepted}; status: interface unavailable; \(T\) unreported.
\EndIf
\State \(s_k^\star\gets\sigma(a^\top z_k^\star+b)\) for \(k=1,\ldots,m\).
\Comment{Owner-side computation}
\State \(T\gets\sum_{k=1}^m\mathbf1[\mathbf1[s_k^\star\ge1/2]=w_k]\).
\If{\(T\ge\tau\)}
\State \textbf{return} \texttt{Watermark Match Accepted}, \(T\).
\Else
\State \textbf{return} \texttt{Watermark Match Not Accepted}, \(T\).
\EndIf
\EndProcedure
\end{algorithmic}
\end{algorithm}

\paragraph{Decoder and deployment boundary.}
The fixed coefficients are \(b=0\) and \(a_j=-1+2(j-1)/(C-1)\) for \(C>1\); for a single returned logit, \(a_1=-1\). The logit coordinate order is fixed with the task interface. The coefficients need not be secret, and \(\sum_j a_j=0\) for \(C>1\). They may be stored as nontrainable buffers during training; the deployment protocol serves the task backbone and task head, while the owner retains the \(C+1\) coefficients and applies the decoder locally. Unused auxiliary modules in a shared training container are not required by this interface. No trainable decoder or dedicated watermark endpoint is required. The reported task-logit evaluation covers graph classification; hard-label-only and node-classification variants remain untested.

\paragraph{Private and pseudo-carrier training.}
Embedding uses the owner's invariant targets, not arbitrary independent bit labels. The auxiliary pseudo-carrier attack instead fixes a separately seeded pool of 128 carriers, with owner/pseudo WL collisions excluded by the evaluator. It optimizes
\(\mathcal L_{\mathrm{task}}+0.5\,|B_{\mathrm{pseudo}}|^{-1}\sum_{G\in B_{\mathrm{pseudo}}}(s_\theta(G)-0.5)^2\)
using task batches of 64 and pseudo-carrier batches of 8, Adam for 100 epochs, learning rate \(0.002\), and weight decay \(5\times10^{-4}\). The optimizer receives no owner carriers or key. The original owner audit material remains fixed, and the one-percentage-point task-utility criterion is assessed after editing. Appendix~\ref{app:adaptive-attack} reports this auxiliary-interface experiment; it does not establish task-logit robustness to the same attack.

\paragraph{Scope of null calibration.}
Carrier-induced keys need not be balanced or independent. The nominal false-positive bound requires the stated null-mean condition and Assumption~\ref{ass:rho-mix}, applied separately to each detector. Finite pooled-null diagnostics in Appendix~\ref{app:original-comparisons} do not establish these assumptions or a binomial or beta-binomial law for the within-audit match count. Confidence intervals for repeated complete-audit outcomes address a different sampling unit. All thresholds remain fixed during subsequent edits and audits.

\clearpage
\endgroup
\section{Imperceptibility: Full Proof and Calibration}
\label{sec:app-impercept}

This appendix proves Theorem~\ref{thm:impercept} and separates its exact hypotheses from the finite calibration diagnostics.

\subsection{Setting and local regularity}

Write
\[
J(\theta)=\mathcal L_{\mathrm{task}}(\theta)+\beta_{\mathrm{wm}}\mathcal L_{\mathrm{wm}}(\theta),
\qquad
\mathcal L_{\mathrm{wm}}(\theta)=\frac1m\sum_{k=1}^m(s_\theta(G_W^{(k)})-y_k)^2,
\]
where \(y_k:=\tilde\lambda_2(G_W^{(k)})\in[0,1]\) is fixed supervision and \(s_\theta(G)\in[0,1]\). Let \(\theta^\star\) be the task-only reference and let
\(\mathcal N:=\{\theta:\|\theta-\theta^\star\|_2\le\delta\}\).
Assume that the task loss satisfies
\begin{equation}
\label{eq:pl-appendix}
\frac12\|\nabla_\theta\mathcal L_{\mathrm{task}}(\theta)\|_2^2
\ge\mu_{\mathrm{PL}}\bigl(\mathcal L_{\mathrm{task}}(\theta)-\mathcal L_{\mathrm{task}}(\theta^\star)\bigr),
\qquad \theta\in\mathcal N.
\end{equation}
Assume \(s_\theta\) is parameter-Lipschitz with constant \(L_s\) in that neighborhood and differentiable at the point under consideration. The trained point \(\tilde\theta\) must belong to \(\mathcal N\) and satisfy \(\nabla J(\tilde\theta)=0\). These conditions hold for a differentiable interior minimizer; a constrained boundary optimum need not satisfy the same stationarity identity. No boundary-optimum guarantee is asserted.

\subsection{Auxiliary lemmas}

\begin{lemma}[Gradient of the watermark loss]
\label{lem:wm-grad}
At a differentiable point,
\[
\nabla_\theta\mathcal L_{\mathrm{wm}}(\theta)
=\frac2m\sum_{k=1}^m
(s_\theta(G_W^{(k)})-y_k)\nabla_\theta s_\theta(G_W^{(k)}).
\]
\end{lemma}
\begin{proof}
Apply the chain rule to each squared-error term and average over carriers.
\end{proof}

\begin{lemma}[Watermark-gradient bound]
\label{lem:grad-bound-app}
Under the parameter-Lipschitz condition and bounded scores and targets,
\(\|\nabla_\theta\mathcal L_{\mathrm{wm}}(\theta)\|_2\le2L_s\).
\end{lemma}
\begin{proof}
Differentiability and parameter Lipschitzness imply
\(\|\nabla_\theta s_\theta(G)\|_2\le L_s\). By Lemma~\ref{lem:wm-grad},
\[
\|\nabla_\theta\mathcal L_{\mathrm{wm}}(\theta)\|_2
\le\frac2m\sum_{k=1}^m|s_\theta(G_W^{(k)})-y_k|\,L_s
\le2L_s,
\]
because each absolute residual is at most one.
\end{proof}

\begin{lemma}[Task gradient at a stationary point]
\label{lem:fooc}
If \(\nabla J(\tilde\theta)=0\), then
\[
\nabla_\theta\mathcal L_{\mathrm{task}}(\tilde\theta)
=-\beta_{\mathrm{wm}}\nabla_\theta\mathcal L_{\mathrm{wm}}(\tilde\theta).
\]
\end{lemma}
\begin{proof}
Expand the gradient of the joint objective and rearrange.
\end{proof}

\subsection{Proof of Theorem~\ref{thm:impercept}}

\begin{proof}
Lemmas~\ref{lem:grad-bound-app} and~\ref{lem:fooc} give
\(\|\nabla\mathcal L_{\mathrm{task}}(\tilde\theta)\|_2\le2\beta_{\mathrm{wm}}L_s\).
Since \(\tilde\theta\in\mathcal N\), the PL inequality applies at \(\tilde\theta\), yielding
\[
\mathcal L_{\mathrm{task}}(\tilde\theta)-\mathcal L_{\mathrm{task}}(\theta^\star)
\le\frac{\|\nabla\mathcal L_{\mathrm{task}}(\tilde\theta)\|_2^2}{2\mu_{\mathrm{PL}}}
\le\frac{2\beta_{\mathrm{wm}}^2L_s^2}{\mu_{\mathrm{PL}}}.
\]
The right side is at most \(\varepsilon_{\mathrm{task}}\) whenever
\begin{equation}
\label{eq:beta-max-appendix}
\beta_{\mathrm{wm}}\le\beta_{\max}:=
\frac{\sqrt{\mu_{\mathrm{PL}}\varepsilon_{\mathrm{task}}}}{\sqrt2 L_s}.
\end{equation}
\end{proof}

The theorem is conditional on local stationarity and regularity. It neither proves that training enters the PL neighborhood nor that flipping a watermark must increase task loss. An upper bound on task loss after embedding cannot be reversed into a lower bound on task degradation after removal.

\subsection{Calibration of PL and Lipschitz quantities}

\paragraph{PL diagnostic.}
For a checkpoint and task-only reference evaluated on the same task batch, define the finite loss gap \(g=\mathcal L_{\mathrm{task}}(\theta)-\mathcal L_{\mathrm{task}}(\theta^\star)\). When \(g>0\), the implementation records the local ratio
\[
\hat\mu_{\mathrm{ratio}}(\theta)
=\frac{\|\nabla\mathcal L_{\mathrm{task}}(\theta)\|_2^2}{2g}.
\]
The factor of two matches \eqref{eq:pl-appendix}; a nonpositive gap does not yield this positive-gap diagnostic. This is a checkpoint-and-batch measurement, not a regression confidence limit or an estimate of the infimum over a neighborhood. It uses full-batch task loss for node classification and a fixed task batch for graph classification.

The reported PL quantities and normalized-weight results in Table~\ref{tab:pl_sensitivity} remain finite applicability diagnostics. In particular, neither a pointwise ratio nor the distance ratio \(\|\tilde\theta-\theta^\star\|_2/\delta\) proves the PL inequality throughout \(\mathcal N\). A diagnostic radius is a chosen reference scale; it is not a certified basin of regularity.

\paragraph{Sensitivity diagnostic.}
On sampled training/carrier graphs, automatic differentiation gives
\(\|\nabla_\theta s_\theta(G)\|_2\). The diagnostic \(\hat L_s\) is the maximum over the sampled graphs, with fixed empirical buffer
\(L_{s,\mathrm{emp}}=(1+\epsilon_L)\hat L_s\), \(\epsilon_L=0.12\).
This finite statistic does not certify the theoretical \(L_s\) for all graphs and parameter paths.

\paragraph{Weight and applicability.}
For interpreting the theorem, the plug-in reference scale is
\[
\beta_{\max,\mathrm{emp}}=
\frac{\sqrt{\hat\mu_{\mathrm{PL}}\varepsilon_{\mathrm{task}}}}{\sqrt2 L_{s,\mathrm{emp}}},
\]
where the tolerance \(\varepsilon_{\mathrm{task}}\) measures an increase in task loss; an accuracy tolerance is a separate empirical criterion. The implementation's default watermark-loss coefficient is the absolute value \(\beta_{\mathrm{wm}}=0.8\), frozen before evaluation. It is not an automatic implementation of a \(\beta_{\max,\mathrm{emp}}\)-constrained search. Table~\ref{tab:pl_sensitivity} explicitly labels its separately reported normalized-weight diagnostic. Finite weight and distance checks do not establish exact stationarity or uniform PL/Lipschitz validity, and a plug-in loss bound retains those mathematical conditions.

\section{Robustness: Full Proof and Calibration}
\label{app:robust-full}

The decision-preservation result is deterministic. The false-positive bound additionally requires the stated pooled-null mean and joint concentration condition.

\subsection{Decision preservation under bounded drift}

For fixed carriers, let
\[
\bar w_k:=\mathbf1[s_{\tilde\theta}(G_W^{(k)})\ge1/2],
\qquad
T_0:=\sum_{k=1}^m\mathbf1[\bar w_k=w_k].
\]
Define the clean carrier margin by
\begin{equation}
\label{eq:decision-margin-app}
\kappa_{\mathrm{marg}}(\tilde\theta)
:=\min_{k\in[m]}|s_{\tilde\theta}(G_W^{(k)})-1/2|.
\end{equation}

\begin{lemma}[Decision preservation]
\label{lem:sign-preserve}
Suppose the required score interface is available before and after the edit. If
\[
\gamma(\hat\theta;\tilde\theta)
:=\max_{G\in\mathcal G_W}|s_{\hat\theta}(G)-s_{\tilde\theta}(G)|
<\kappa_{\mathrm{marg}}(\tilde\theta),
\]
then every decoded bit is unchanged and \(T(\hat\theta)=T_0\).
\end{lemma}
\begin{proof}
Each clean score is at least \(\kappa_{\mathrm{marg}}(\tilde\theta)\) away from threshold \(1/2\). A strictly smaller change cannot cross the threshold. Therefore the decoded bit, and hence its match indicator, is unchanged on every carrier.
\end{proof}
If \(T_0\ge\tau\), the attacked model passes the same deterministic test. The statement preserves the clean match count even when some clean bits are wrong. A zero margin yields no nontrivial guarantee. At or above the margin, the lemma is silent. A missing audit endpoint has no score to compare and is an interface failure outside the drift condition.

\subsection{Audit preservation when some carrier decisions may change}
\label{app:partial-margin}

Fix the carrier set, owner key, clean checkpoint \(\tilde\theta\), attacked checkpoint \(\hat\theta\), deterministic threshold decoder, and acceptance threshold \(\tau\). Assume the same score interface is available on every carrier before and after the edit. Define
\[
\mathcal C_0:=\{k:\widehat w_k(\tilde\theta)=w_k\},\qquad
T_0:=|\mathcal C_0|,\qquad
r_k:=|s_{\tilde\theta}(G_W^{(k)})-1/2|,
\]
and
\[
u_k:=|s_{\hat\theta}(G_W^{(k)})-s_{\tilde\theta}(G_W^{(k)})|.
\]
Here \(r_k\) is the clean \emph{score} margin. It is distinct from the target margin
\(|\tilde\lambda_2(G_W^{(k)})-1/2|\); the argument does not substitute a target margin for a score margin.

\begin{proof}[Proof of \eqref{eq:partial-margin-audit}]
If \(k\in\mathcal C_0\) and \(u_k<r_k\), the attacked score cannot cross threshold \(1/2\), so the attacked decoded bit still equals the owner bit. Consequently,
\[
\begin{aligned}
T(\hat\theta)
&=\sum_{k=1}^m\mathbf1[\widehat w_k(\hat\theta)=w_k]\\
&\ge\sum_{k\in\mathcal C_0}\mathbf1[u_k<r_k]\\
&=T_0-\sum_{k\in\mathcal C_0}\mathbf1[u_k\ge r_k].
\end{aligned}
\]
The first inequality discards any matches gained on initially incorrect carriers, so it remains valid when such gains occur. The weak inequality in the final indicator conservatively treats drift exactly equal to the margin as potentially unsafe. A carrier with \(r_k=0\) is also counted as potentially unsafe, including when its actual decision happens to be unchanged.
\end{proof}

When \(T_0\ge\tau\), at most \(T_0-\tau\) potential lost matches can be tolerated in this bound. If certified nonnegative bounds \(\bar u_k\ge u_k\) are available, replacing \(u_k\) by \(\bar u_k\) can only increase the counted set and therefore yields another valid sufficient condition. A uniform bound \(\gamma\ge\max_k u_k\) gives
\[
T(\hat\theta)\ge T_0-\bigl|\{k\in\mathcal C_0:r_k\le\gamma\}\bigr|.
\]
Equivalently, for \(1\le\tau\le T_0\), let
\(r_{(1)}\le\cdots\le r_{(T_0)}\)
be the ordered score margins on the initially correct carriers. The condition
\(r_{(T_0-\tau+1)}>\gamma\)
is sufficient for acceptance. The original minimum-margin condition is a stronger sufficient condition that permits no decoded-bit changes; the present bound can retain acceptance while allowing some changes.

All quantities must refer to the same fixed checkpoint pair, decoder, and carriers. A bound for an attack class additionally requires that the chosen drift bounds hold for every attack in that class. Seed means, watermark-accuracy averages, and separately aggregated drift summaries do not establish these per-carrier conditions. This deterministic result neither supplies a new null model nor changes the Type-I error conditions for the pre-specified threshold.

\subsection{Composite budget inequality}

\begin{lemma}[Conditional budget decomposition]
\label{lem:budget}
Consider the path
\(\tilde\theta\to\theta^{\mathrm{ft}}\to\theta^{\mathrm{ft,pr}}\to\hat\theta\).
Assume the fine-tuning segment lies in the parameter-Lipschitz neighborhood and define
\(\Delta_\theta:=\|\theta^{\mathrm{ft}}-\tilde\theta\|_2\).
For pruning fraction \(p_{\mathrm{pr}}\) and distillation-strength parameter
\(\pi_{\mathrm{kd}}:=1-\rho_{\mathrm{kd}}\), assume
\[
\max_{G\in\mathcal G_W}|s_{\theta^{\mathrm{ft,pr}}}(G)-s_{\theta^{\mathrm{ft}}}(G)|
\le c_{\mathrm{prune}}\sqrt{p_{\mathrm{pr}}},
\qquad
\max_{G\in\mathcal G_W}|s_{\hat\theta}(G)-s_{\theta^{\mathrm{ft,pr}}}(G)|
\le c_{\mathrm{distill}}\pi_{\mathrm{kd}},
\]
where \(\rho_{\mathrm{kd}}\in[0,1]\) is the retention parameter used to index the distillation sweep. Then \eqref{eq:budget} holds.
\end{lemma}
\begin{proof}
For every carrier \(G\), apply the triangle inequality along the three score changes:
\[
\begin{aligned}
|s_{\hat\theta}(G)-s_{\tilde\theta}(G)|
&\le |s_{\theta^{\mathrm{ft}}}(G)-s_{\tilde\theta}(G)|
+|s_{\theta^{\mathrm{ft,pr}}}(G)-s_{\theta^{\mathrm{ft}}}(G)|\\
&\quad+|s_{\hat\theta}(G)-s_{\theta^{\mathrm{ft,pr}}}(G)|.
\end{aligned}
\]
The first term is at most \(L_s\Delta_\theta\) under the local parameter bound. Bound the other terms by the two stated hypotheses, then maximize over carriers. No parameter-distance comparison between architectures is needed for the distillation term; the two score functions must nevertheless be available on the same carriers.
\end{proof}

\subsection{Joint concentration and Chernoff's bound}
\label{app:rho-mix-blocking}

Let \(U_1,\ldots,U_m\in[0,1]\) satisfy Assumption~\ref{ass:rho-mix}, and put
\(S:=\sum_{k=1}^m(U_k-\mathbb E U_k)\).
For \(t>0\), Markov's inequality and \eqref{eq:joint-mgf} give
\[
\Pr[S\ge m\delta]
\le \exp\!\left(-tm\delta+\frac{mt^2}{8(1-c_{\mathrm{dep}})}\right).
\]
The choice \(t=4(1-c_{\mathrm{dep}})\delta\) yields
\begin{equation}
\label{eq:rho-mix-hoeffding}
\Pr\!\left[\frac1m\sum_{k=1}^m(U_k-\mathbb E U_k)\ge\delta\right]
\le \exp\{-2(1-c_{\mathrm{dep}})m\delta^2\}.
\end{equation}
Applying the same argument to \(-S\) gives the lower-tail bound. For independent indicators, Hoeffding's lemma supplies \eqref{eq:joint-mgf} with \(c_{\mathrm{dep}}=0\). Small pairwise correlations alone are insufficient: they constrain second moments, not the full joint moment-generating function. No general conversion \(c_{\mathrm{dep}}\le4\rho_0\) is asserted.

\paragraph{Application to the pooled null.}
For a fixed, pre-specified detector, take
\(X_k:=\mathbf1[\hat w_k=w_k]\).
Under the specified pooled-null law \(H_0\), assume
\(\mathbb E[X_k\mid H_0]=1/2\) and the conditional version of \eqref{eq:joint-mgf}. With
\(\delta_\tau:=\tau/m-1/2>0\), equation~\eqref{eq:rho-mix-hoeffding} gives
\begin{equation}
\label{eq:alpha-bound-app}
\Pr[T\ge\tau\mid H_0]
\le\exp\{-2(1-c_{\mathrm{dep}})m\delta_\tau^2\}.
\end{equation}
The probability is under that pooled null. A calibrated pooled mean does not establish the same mean or bound for every individual null model. If the threshold is written
\(\tau=\lceil m(1-\varepsilon_{\mathrm{err}})\rceil\), the exponent is determined by
\(\delta_\tau=\tau/m-1/2\), which differs from \(1/2-\varepsilon_{\mathrm{err}}\) only through rounding.

\subsection{False negatives under bounded drift}
\label{app:beta-fn}

\paragraph{Deterministic decoding.}
Lemma~\ref{lem:sign-preserve} implies \(\beta_{\mathrm{fn}}=0\) conditional on
\(T_0\ge\tau\), an available score interface, and
\(\gamma(\hat\theta;\tilde\theta)<\kappa_{\mathrm{marg}}(\tilde\theta)\).
For a random attack, the drift condition must hold for the realizations being covered; it is not enough that the average drift is below an average margin. A per-run certificate requires the maximum score drift and minimum clean margin on the same checkpoint pair and carrier set.

\paragraph{Optional stochastic decoding.}
Let \(Z_k\) indicate that inference-time randomness changes the clean decoded bit on carrier \(k\). For
\(a:=\kappa_{\mathrm{marg}}(\tilde\theta)-\gamma(\hat\theta;\tilde\theta)>0\),
suppose \(\Pr[Z_k=1]\le\psi(a)\), and separately assume the joint concentration condition for the sequence \(Z_1,\ldots,Z_m\).
Since \(T(\hat\theta)\ge T_0-\sum_k Z_k\), define
\(\Delta_{\mathrm{own}}:=(T_0-\tau)/m\).
If \(\Delta_{\mathrm{own}}>\psi(a)\), then
\begin{equation}
\label{eq:beta-bound-app}
\beta_{\mathrm{fn}}\le
\exp\{-2(1-c_{\mathrm{dep}})m(\Delta_{\mathrm{own}}-\psi(a))^2\}.
\end{equation}
For deterministic decoding, the direct lemma gives the stronger zero-error conclusion under its conditions; substituting \(\psi(a)=0\) in the exponential expression alone does not make that expression zero.

\subsection{Calibration of edit constants and verification threshold}
\label{app:robust-calib}

\paragraph{Empirical edit slopes.}
The reported validation sweeps use
\[
\widehat c_{\mathrm{prune}}
=\max_{p\in\{0.2,0.4,0.5\}}
\frac{\gamma(\theta^{\mathrm{ft,pr}}(p);\theta^{\mathrm{ft}})}{\sqrt p},
\qquad
\widehat c_{\mathrm{distill}}
=\max_{\pi\in\{0.25,0.5,0.75,1.0\}}
\frac{\gamma(\hat\theta(\pi);\theta^{\mathrm{ft,pr}}(0.5))}{\pi}.
\]
Substituting these slopes in \eqref{eq:budget} gives a finite-sweep diagnostic. Such maxima do not establish the required bounds for attacks, budgets, datasets, or checkpoints outside the measured sweep. In particular, no linear-in-\(\pi\) bound at \(\pi=0\) follows from a sweep that begins at \(0.25\).

\paragraph{Threshold calibration.}
The statistic-bank protocol in Appendix~\ref{sec:appendix-protocol-rho} reports \(\hat\rho_0\); score and invariant correlations are additional diagnostics. The original plug-in convention is
\(\widehat c_{\mathrm{dep}}:=\min\{4\hat\rho_0,0.5\}\).
For a factor \(c_{\mathrm{dep}}\) satisfying the joint condition and target rate \(\alpha\), set
\begin{equation}
\label{eq:delta-alpha-app}
\delta_\alpha:=
\sqrt{\frac{\log(1/\alpha)}{2(1-c_{\mathrm{dep}})m}},
\end{equation}
and
\begin{equation}
\label{eq:tau-alpha-app}
\tau:=\left\lceil m\left(\frac12+\delta_\alpha\right)\right\rceil.
\end{equation}
If \(\tau>m\), no observed count can be accepted under this conservative bound. The equivalent error-tolerance parameter is
\(\varepsilon_{\mathrm{err}}:=1/2-\delta_\alpha\).
Replacing \(c_{\mathrm{dep}}\) by its empirical plug-in value does not establish that the mathematical assumption holds.

\paragraph{Worked example.}
Retaining the reported values \(m=128\), \(\alpha=10^{-6}\), and
\(\hat\rho_0=7.6\times10^{-4}\) gives
\(\widehat c_{\mathrm{dep}}=3.04\times10^{-3}\) and
\[
\delta_\alpha=
\sqrt{\frac{\log(10^6)}{2(1-3.04\times10^{-3})\,128}}
\approx0.2327,\qquad
\tau=\lceil128(0.5+0.2327)\rceil=94.
\]
Thus \(\varepsilon_{\mathrm{err}}\approx0.2673\).
The scalar and task-logit detectors are calibrated separately before attack evaluation. The concentration guarantee is for one pre-specified test. Uncorrected searches over carriers, keys, thresholds, checkpoints, or decoder coefficients are not covered.

\section{Uniqueness: Full Proof and Calibration}
\label{app:uniq-full}

Theorem~\ref{thm:uniq} separates independently drawn carrier-induced keys and the conditional decoding laws concentrated near those keys. It does not by itself establish that every unrelated model fails an owner's audit.

\subsection{Setup and probability spaces}

For realized owner carriers \(\mathcal G_W=\{G_W^{(k)}\}_{k=1}^m\), define
\[
w_k:=\mathbf1[\tilde\lambda_2(G_W^{(k)})\ge1/2],
\qquad
b(W):=(\mathbf1[s_{\tilde\theta}(G_W^{(k)})\ge1/2])_{k=1}^m.
\]
Let \(F_W:=\mathrm{Law}(b(W)\mid\mathcal G_W,W)\), with randomness in the training procedure or audit mechanism after conditioning on the realized carriers and key. The notation suppresses the particular carrier set. For a fixed checkpoint and deterministic decoding, \(F_W\) is a point mass. The outer probability in the theorem concerns the independent sampling of two owner carrier sets, not fresh sampling inside \(F_W\).

For each carrier position, let
\[
p:=\Pr_{G\sim\mathrm{protocol}}[\tilde\lambda_2(G)\ge1/2],
\qquad
q:=2p(1-p).
\]
We assume the same marginal \(p\) across positions and
\(p\in[p_{\min},1-p_{\min}]\), where \(p_{\min}\in(0,1/2]\).
Consequently \(q\ge q_{\min}:=2p_{\min}(1-p_{\min})\).
Write \(d_H(x,y):=\sum_k\mathbf1[x_k\ne y_k]\) for Hamming distance.

\subsection{Inter-owner key separation}

For independent carrier sets with keys \(W,W'\), define
\(Y_k:=\mathbf1[w_k\ne w'_k]\).
Independence between the two sets gives \(\mathbb E Y_k=q\). It does not alone give independence across positions; Assumption~\ref{ass:rho-mix} is additionally imposed on the sequence \(Y_k\).
For \(0\le\eta_{\mathrm{dec}}<q_{\min}/2\), the lower-tail version of \eqref{eq:rho-mix-hoeffding} yields
\begin{equation}
\label{eq:key-separation-app}
\Pr[d_H(W,W')\le2m\eta_{\mathrm{dec}}]
\le\exp\{-2(1-c_{\mathrm{dep}})m(q_{\min}-2\eta_{\mathrm{dec}})^2\}.
\end{equation}

\subsection{Conditional decoding accuracy}

For each realized carrier set and key, assume
\[
F_W(\mathcal B_W)\ge1-\delta_{\mathrm{dec}},
\qquad
\mathcal B_W:=\{x:d_H(x,W)\le m\eta_{\mathrm{dec}}\}.
\]
The same upper error-probability bound must hold for both owners' conditional laws. For a fixed checkpoint, deterministic decoding gives zero conditional error probability if its measured error count is within the chosen radius. A measured error rate on one key does not establish a uniform statement over all independently sampled keys or training runs. The theorem assumes the needed conditional decoding accuracy; any empirical estimate must respect the relevant probability space.

\subsection{Total-variation separation}

On the event \(d_H(W,W')>2m\eta_{\mathrm{dec}}\), the balls
\(\mathcal B_W\) and \(\mathcal B_{W'}\) are disjoint by the triangle inequality.
Therefore \(F_{W'}(\mathcal B_W)\le\delta_{\mathrm{dec}}\), and
\[
\mathrm{TV}(F_W,F_{W'})
:=\sup_{A\subseteq\{0,1\}^m}|F_W(A)-F_{W'}(A)|
\ge F_W(\mathcal B_W)-F_{W'}(\mathcal B_W)
\ge1-2\delta_{\mathrm{dec}}.
\]
Together with \eqref{eq:key-separation-app}, this proves Theorem~\ref{thm:uniq}.
If additionally \(\delta_{\mathrm{dec}}\le\exp(-c_{\mathrm{dec}}m)\) for a constant
\(c_{\mathrm{dec}}>0\), the separation is \(1-\exp(-\Omega(m))\).
The rates depend on key balance, the decoding-error radius, the joint-concentration factor, and the decay of decoding error. Key-string separation is distinct from a false-positive guarantee against an arbitrary model; that guarantee requires the detector-specific null conditions of Theorem~\ref{thm:robust}.

\subsection{Calibration of key balance and decoding tolerance}
\label{app:uniq-calib}

\paragraph{Two-sided key-balance requirement.}
For \(N\) independent candidate draws, let \(n_1\) count those with
\(\tilde\lambda_2(G)\ge1/2\), and let \(n_0=N-n_1\). Then \(\hat p=n_1/N\).
To obtain simultaneous lower bounds on both \(p\) and \(1-p\), choose calibration error budget \(\delta_{\mathrm{bal}}\in(0,1)\) and define
\[
\begin{aligned}
\ell_1&=\mathrm{BetaInv}(\delta_{\mathrm{bal}}/2;n_1,n_0+1),\\
\ell_0&=\mathrm{BetaInv}(\delta_{\mathrm{bal}}/2;n_0,n_1+1),\qquad
p_{\min}:=\min\{\ell_1,\ell_0,1/2\}.
\end{aligned}
\]
Here \(\mathrm{BetaInv}(u;a,b)\) is the \(u\)-quantile of a beta distribution with parameters \(a,b\); set \(\ell_j=0\) when \(n_j=0\).
A union bound on the two one-sided Clopper--Pearson intervals gives
\(p\in[p_{\min},1-p_{\min}]\) with confidence at least \(1-\delta_{\mathrm{bal}}\).
Binomial confidence coverage requires the independent candidate-draw model; adaptive or dependent candidates need separate calibration. If the theorem is applied with an estimated interval, its unconditional failure probability also includes the calibration error \(\delta_{\mathrm{bal}}\).

\paragraph{Decoding tolerance.}
For a fixed deterministic checkpoint, the observed clean error fraction is
\(\hat\eta_{\mathrm{dec}}:=1-T_0/m\).
A fixed-key statement can choose a radius at least this large. A statement over random keys and training runs requires a corresponding conditional error bound, not merely a seed mean. In all cases the separation argument requires
\[
2\eta_{\mathrm{dec}}<q_{\min}=2p_{\min}(1-p_{\min}).
\]
If this condition fails, increasing \(m\) alone does not repair the separation constant; key balance or decoding accuracy must improve.

\section{Worst-Case Exact-Removal Hardness: Full Problem and Proof}
\label{app:unremovable-full}

This appendix proves Theorem~\ref{thm:unremovable} for a restricted rational linear-threshold decoder class. The result concerns exact sparse removal in the worst case, not security of an arbitrary trained GNN.

\subsection{Formal decision problem}

\begin{definition}[Problem \textsc{WM--Remove}]
\label{def:wm-remove}
Inputs are a rational vector \(\tilde\theta\in\mathbb Q^d\), nonnegative rational coefficients \(A\in\mathbb Q_+^{m\times d}\), rational thresholds \(b\in\mathbb Q^m\), and an integer budget \(B\). For a fixed rational amplitude \(\vartheta_{\min}>0\), ask whether there exist a support \(\mathcal J\subseteq[d]\), \(|\mathcal J|\le B\), and a perturbation \(\Delta\theta\) satisfying
\[
\Delta\theta_j=0\quad(j\notin\mathcal J),\qquad
|\Delta\theta_j|\ge\vartheta_{\min}\quad(j\in\mathcal J),
\qquad
\operatorname{Dec}_k(\tilde\theta+\Delta\theta)=1-\operatorname{Dec}_k(\tilde\theta)
\quad(k\in[m]),
\]
where
\begin{equation}
\label{eq:monotone-decoder-app}
\operatorname{Dec}_k(\theta)=\mathbf1[g_k(\theta)\ge b_k],
\qquad
g_k(\theta):=\sum_{j=1}^d a_{kj}\theta_j.
\end{equation}
All input rationals use finite binary encodings.
\end{definition}

The carriers index the bits and do not otherwise enter the reduction. The linear-threshold map is coordinate-wise nondecreasing because \(a_{kj}\ge0\); supports may overlap across bits. The problem requires every bit to flip and includes no task-utility constraint.

\subsection{Reduction from \textsc{Set Cover}}

\paragraph{Source problem.}
An instance consists of a universe \(U=\{u_1,\ldots,u_m\}\), subsets
\(\mathcal C=\{C_1,\ldots,C_q\}\), and integer \(B\). The NP-complete decision problem \textsc{Set Cover} asks whether at most \(B\) of the subsets cover every element of \(U\).

\paragraph{Construction.}
Set \(d=q\), with coordinate \(j\) indexed by \(C_j\), and choose
\[
a_{kj}:=\mathbf1[u_k\in C_j],\qquad
b_k:=\vartheta_{\min}/2,\qquad
\tilde\theta:=0.
\]
The constructed matrix has \(mq\) binary entries and every initial decoded bit is zero. For fixed rational \(\vartheta_{\min}>0\), the construction has polynomial encoding length and can be computed in polynomial time.

\paragraph{Forward direction.}
Suppose \(H\subseteq[q]\) indexes a cover of size at most \(B\). Set
\(\Delta\theta_j=\vartheta_{\min}\) for \(j\in H\), and zero otherwise. Every element \(u_k\) belongs to a selected set, so
\[
g_k(\tilde\theta+\Delta\theta)
=\sum_j a_{kj}\Delta\theta_j
\ge\vartheta_{\min}>\vartheta_{\min}/2=b_k.
\]
Every bit flips, the support size is at most \(B\), and all nonzero amplitudes satisfy the requirement.

\paragraph{Reverse direction.}
Suppose an admissible perturbation flips all bits. For each \(k\),
\[
\sum_{j\in\mathcal J}a_{kj}\Delta\theta_j\ge\vartheta_{\min}/2>0.
\]
Because \(a_{kj}\ge0\), at least one selected coordinate with \(a_{kj}=1\) must have a positive update. The positive-support indices
\(\mathcal J_+:=\{j\in\mathcal J:\Delta\theta_j>0\}\)
therefore index subsets covering every element of \(U\). Their number is at most \(B\), proving equivalence of the two decision instances and NP-hardness.

\subsection{Membership in NP}

Guess the support \(\mathcal J\), of size at most \(B\), and the sign of each selected update. These guesses use polynomially many bits. Once support and signs are fixed, all remaining constraints are rational linear inequalities: positive selected updates are at least \(\vartheta_{\min}\), negative ones are at most \(-\vartheta_{\min}\), initially zero bits must reach their thresholds, and initially one bits must fall strictly below their thresholds.

Strict feasibility can be checked by adding a common slack \(u\), \(0\le u\le1\), to the last group of inequalities and maximizing \(u\). A strictly feasible solution exists exactly when the optimum is positive; if that group is empty, ordinary linear feasibility suffices. Rational linear programming is polynomial-time, so a nondeterministic polynomial-time verifier can check the guessed support/sign pattern. Equivalently, a feasible rational linear system admits a solution of polynomial encoding length. Thus \textsc{WM--Remove} is in NP. Combining membership with the reduction proves Theorem~\ref{thm:unremovable}.

\subsection{Decoder scope and practical interpretation}

A nonnegative linear head is monotone in its inputs, but a nonlinear GNN is not thereby linear or coordinate-wise monotone in all trainable parameters. Nonnegative last-layer weights, projection, or a penalty therefore do not provide a checkpoint-level certificate that the full model lies in the class of \eqref{eq:monotone-decoder-app}. No such certificate is claimed for the implemented GNN or task-logit variant.

Worst-case NP-completeness does not show that practical removal of a given trained instance is hard. The theorem also does not cover approximate removal, endpoint deletion, changing the decoder class, or flipping only enough bits to make \(T<\tau\). Practical conclusions are restricted to the measured attacks and budgets. The separate margin theorem explains only those edits whose actual score drift is strictly below the clean carrier margin.

\subsection{Invariant choice and carrier-evaluation cost}
\label{app:lambda-choice-complexity}

Algebraic connectivity \(\lambda_2(G)\) is a permutation-invariant measure of global connectivity. The normalized target \(\tilde\lambda_2(G)\in[0,1]\) fits the scalar score and threshold decoder. Other bounded permutation-invariant targets can use the same conditional analysis if their score, carrier, and decoder assumptions hold. The reported choice of \(\lambda_2\) is matched to the degree-preserving rewiring protocol; it is not claimed to be a universal semantic invariant.

A dense eigensolver costs \(O(n^3)\) for an \(n\)-node carrier. The reported default carriers satisfy \(n\le n_{\mathrm{cap}}=\max\{4,\min\{32,\lfloor n_{0.25}\rfloor\}\}\) and \(m\le256\). For larger carriers, Lanczos or an appropriately shifted/deflated iterative method can estimate the second-smallest eigenvalue. Eigenvalue computation is owner-side preprocessing; gradients do not pass through it, and suspect-model verification requires only the pre-specified score queries.

\section{Experimental Details}
\label{app:exp_details}

\subsection{Experimental Setup}
\label{app:setup}

\paragraph{Tasks, datasets, and backbones.}
We evaluate \textbf{InvGNN-WM} on both node- and graph-level classification.
Node datasets: \textbf{Cora}, \textbf{PubMed}~\citep{Sen2008CoraPubMed,Yang2016Planetoid}, \textbf{Amazon-Photo}~\citep{Shchur2019Pitfalls}.
Graph datasets: \textbf{PROTEINS}, \textbf{NCI1}~\citep{Morris2020TUDataset}.
Backbones: node-level \textbf{GCN}~\citep{KipfWelling2017GCN}, \textbf{GraphSAGE}~\citep{Hamilton2017GraphSAGE}, \textbf{SGC}~\citep{Wu2019SGC};
graph-level \textbf{GIN}~\citep{xu2019how}, \textbf{GraphSAGE}.
Unless otherwise stated, we run 100 epochs with Adam (lr $=0.01$), seed set $\{41,42,43\}$, and report mean $\pm$ 95\% CI.

\paragraph{Watermark configuration.}
We embed $m{=}128$ bits. Owner-private carriers $\mathcal{G}_W$ are generated by degree-preserving double-edge swaps with two checks: (i) empirical graph non-collision via WL hashes; (ii) pool-level KS checks on mean degree, mean clustering, and spectral radius (p$\ge\delta$, $\delta{=}0.1$), active once the augmented pool contains at least five carriers. Swap steps increase by 5 up to 50; acceptance requires all active checks. Every carrier satisfies $|V(G_W)|\le n_{\mathrm{cap}}:=\max\{4,\min\{32,\lfloor n_{0.25}\rfloor\}\}$, where $n_{0.25}$ is the $25^{\text{th}}$ percentile of task-graph node counts (see Assumption Protocols in Appendix~\S\ref{sec:appendix-assumptions}).
The invariant is instantiated as normalized algebraic connectivity $\tilde{\lambda}_2$ (Section~\ref{sec:method}); the perception head $s_\theta$ is spectrally normalized (target operator norm $\nu{=}1.0$) to regularize score sensitivity.

\paragraph{Verification threshold.}
Given target false-positive rate $\alpha$ and mixing estimate $\hat\rho_0$, set
\[
\widehat c_{\mathrm{dep}}:=\min\{4\hat\rho_0,0.5\},\qquad
\delta_\alpha:=
\sqrt{\frac{\log(1/\alpha)}{2(1-\widehat c_{\mathrm{dep}})m}},
\qquad
\tau(\alpha)=\left\lceil m\left(\frac12+\delta_\alpha\right)\right\rceil .
\]
Equivalently, if $\tau=\lceil m(1-\varepsilon_{\mathrm{err}})\rceil$, then $\varepsilon_{\mathrm{err}}=1/2-\delta_\alpha$ up to the ceiling operation. With $m{=}128$, $\alpha{=}10^{-6}$, and the reported plug-in diagnostic $\hat\rho_0{=}7.6{\times}10^{-4}$, this rule gives $\tau{=}94$. Its theoretical Type-I interpretation additionally requires the pooled-null mean and joint concentration condition in Assumption~\ref{ass:rho-mix}; an empirical pairwise correlation estimate alone does not establish that condition. Scalar-audit and task-logit detectors are calibrated separately before evaluation, and thresholds are not retuned for individual attacks.

\paragraph{Edits (post-hoc modifications).}
The common-edit comparison tests unstructured magnitude pruning (40\%), fine-tuning on clean task data (20 epochs), knowledge distillation (KD, temperature $T{=}2$), and post-training quantization (4-bit). Stronger attacks and pruning followed by retraining are evaluated separately in Table~\ref{tab:interface_attacks}.

\paragraph{Train/val/test splits and reporting.}
For TUD graph datasets we use stratified random 80/10/10 splits per seed with mini-batch training (batch size 64).
Cora and PubMed use the standard Planetoid masks; Amazon-Photo uses a seeded 60/20/20 node split. Node tasks use full-graph training. OGBG-MOLHIV uses its official split, held fixed across the independent training runs in Table~\ref{tab:ogb_compute}.
We select the checkpoint with the best validation task metric (accuracy, or ROC-AUC for OGBG-MOLHIV) and evaluate on the held-out test set.
For clean WM-ACC (including unedited configuration ablations), the four common edits, pseudo-carrier attacks, and fresh-head recovery, we use two-sided nominal $95\%$ Student $t$ intervals over the three seed scores. Let $T_s$ be the correct-bit count for seed $s\in\{41,42,43\}$, let $X_s=100T_s/128$, and define $\bar X=\sum_sX_s/3$ and $S_X^2=\sum_s(X_s-\bar X)^2/2$. We report $\bar X\pm t_{0.975,2}S_X/\sqrt{3}$, where $t_{0.975,2}=4.30265$. Corresponding count intervals apply the same formula to $T_s$. Three-seed Task ACC, task ROC-AUC, and probe $R^2$ summaries with $\pm$ intervals use the same formula on their own seed-level measurements in the reported units. Means and half-widths are rounded only for display. The intervals are untruncated; a zero half-width means that the three observed seed scores coincide, not that future-run uncertainty is zero. Nominal coverage uses an independent-seed, approximately normal sampling model.
For a signed watermark decrease, define $D_s=X_s^{\mathrm{clean}}-X_s^{\mathrm{edited}}$ in percentage points on the same seed, key, and carriers. The paired $t$ interval applies the preceding formula to the three $D_s$ values; a negative decrease denotes improved WM-ACC. Wilson intervals apply to the $N=300$ binary detection trials per setting, and detection AUROC uses the paired hierarchical bootstrap in Appendix~\ref{app:pregip-protocol}. FPR confidence bounds use Clopper--Pearson; diagnostic point estimates and Monte Carlo resolution labels are distinguished from such bounds.
\paragraph{Owner keys and evaluation units.}
For the fixed-owner WM-ACC results, each training seed in $\{41,42,43\}$ generates one owner key of $128$ bits associated with $128$ carrier graphs. The same key and carriers are evaluated after training and under the different attack conditions. For each model and condition, WM-ACC is the fraction of correctly decoded bits on these $128$ carriers; the reported mean averages the three seed-level scores. Evaluating multiple conditions does not resample independent owner keys for a fixed checkpoint. A separate key-randomization diagnostic uses $200$ random impostor bitstrings per seed. The independent-training detection protocol for Table~\ref{tab:ogb_compute} is specified in Appendix~\ref{app:scale-compute}. The owner-instance detection protocol for Table~\ref{tab:pregip_matched} is specified in Appendix~\ref{app:matched-detection-ci}.

\subsection{Metrics}
\label{app:metrics}

\paragraph{Task accuracy (Task ACC).}
Standard top-1 accuracy on the task test set; the OGBG-MOLHIV study reports task ROC-AUC instead.

\paragraph{Watermark fidelity (WM-ACC).}
For each carrier $G_W^{(k)}$, we use the specified scalar or task-logit score and decode $\hat{w}_k=\mathbf{1}[s_\theta(G_W^{(k)})\ge 0.5]$.
WM-ACC is the fraction of correctly recovered bits over the $m$ carriers.
For each reported WM-ACC entry, let $T_j$ be the number of correctly decoded bits
among the $128$ owner carriers for seed $j\in\{41,42,43\}$. The mean percentage is
$100(T_{41}+T_{42}+T_{43})/384$. We keep full precision when averaging and computing
differences, then round to one decimal place, with exact halfway cases rounded upward.

\paragraph{Uniqueness \& calibrated verification.}
The owner's match count $T{=}\sum_{k=1}^m\mathbf{1}[\hat{w}_k{=}w_k]$ is compared with a statistically calibrated threshold $\tau(\alpha)$ (shared across runs via a pooled null). The key-randomization uniqueness check compares the model's decoded bitstring with $200$ random impostor bitstrings per seed. These are wrong-key comparisons on the evaluated models, not separately trained impostor models. The pooled-null calibration and the $10^7$-trial null study in Appendix~\ref{app:original-comparisons} are separate from this key-randomization check.

\paragraph{Verification gap under edits.}
For an edit $e$, let $T_e$ be the owner match count after applying edit $e$. We define the verification gap as
\[
\Delta_{\mathrm{ver}}(e):=T_e-\tau(\alpha).
\]
A nonnegative $\Delta_{\mathrm{ver}}(e)$ means the watermark remains verified after edit $e$. We summarize robustness by $\min_{e\in\mathcal{E}}\Delta_{\mathrm{ver}}(e)$ across the edit set $\mathcal{E}$.

\subsection{Implementation Details}
\label{app:impl}

\paragraph{Environment.}
Experiments use Google Colab with an NVIDIA A100 and CUDA 12.1. Recorded packages are PyTorch 2.2.2, PyG 2.5.3, torch-scatter 2.1.2, torch-sparse 0.6.18, torch-cluster 1.6.3, torch-spline-conv 1.2.2, NumPy 1.26.4, scikit-learn 1.4.2, and NetworkX 3.2.1. We disable nondeterministic CuDNN features and use seeds $\{41,42,43\}$, with derived training and carrier-generation sub-seeds for the detection study in Table~\ref{tab:ogb_compute}. Appendix~\ref{app:assets} lists verified software license information.

\paragraph{Training protocol.}
Optimizer Adam (lr $=0.01$, weight decay $5{\times}10^{-4}$ unless noted), 100 epochs,
gradient clipping off by default, and checkpoint selection by the best validation task metric.
The default hidden width is 64 with four message-passing layers; SGC uses one \texttt{SGConv} with propagation depth $K=2$. GIN and GraphSAGE use global mean pooling for graph-level readout.
GPS uses four attention heads and dropout $0.1$. Node tasks use full-batch training; graph tasks use batch size~64.
The auxiliary scalar head is a linear readout followed by a sigmoid, with spectral normalization at target operator norm $\nu{=}1.0$.
Each task minibatch is paired with a carrier minibatch of size 8 (carrier/task ratio $\leq0.16$ for graph batches).
The default implementation uses absolute watermark weight $\beta_{\mathrm{wm}}=0.8$; the normalized weights in Table~\ref{tab:pl_sensitivity} index the separate local-applicability diagnostic.

\paragraph{Carrier generation and normalization.}
We implement the adaptive two-step sampling from Appendix~\S\ref{sec:appendix-protocol-data} (WL non-collision, KS p$\ge \delta$),
with swap increments of $5$ and cap at $50$.
For invariant normalization (Eq.~\ref{eq:lambda_norm}), $(\lambda_{\min},\lambda_{\text{scale}})$ are the empirical $5$th and $95$th percentiles on the first at most 600 task training graphs. A gap below $10^{-8}$ triggers min--max scaling on that same sample; a still-degenerate range stops construction. Normalization is then frozen.

\paragraph{Mixing estimate and Lipschitz calibration.}
The reported plug-in dependence coefficient is $\hat\rho_0{=}7.6{\times}10^{-4}$ (Appendix~\ref{sec:appendix-protocol-rho}). The finite statistic bank is an empirical diagnostic; it does not certify joint concentration for all functions of the carriers.
The sensitivity routine differentiates the scalar score with respect to all model parameters on a fixed subset of carrier graphs and takes the maximum observed gradient norm as $\hat L_s$. The empirical buffer is $L_{s,\mathrm{emp}}=(1+\epsilon_L)\hat L_s$ with $\epsilon_L=0.12$; its scope is the evaluated graphs and checkpoints.

\paragraph{Verification.}
We query the specified interface on the $m$ carriers, decode bits at $0.5$, compute the watermark match count $T$, and accept a watermark match iff $T\ge\tau(\alpha)$. If the required endpoint is unavailable, the audit cannot support the claim; an unavailable endpoint is not a measured zero match count. The threshold is fixed using disjoint pooled-null calibration for each dataset, backbone, and interface. The default is $\tau=94$.

\paragraph{Edit implementations.}
\emph{Pruning:} one-shot global magnitude pruning at 40\% on eligible linear weight tensors in the common-edit comparison; the stronger-attack study separately evaluates 50\% pruning with retraining.
\emph{Fine-tuning:} 20 epochs on clean task data with the task loss only.
\emph{KD:} a same-architecture student trains for 100 epochs on teacher outputs for task training batches, with temperature $T{=}2$. The loss is KL-divergence for multiclass outputs and soft-target binary cross-entropy for a single-logit task. KD+WM adds owner-side watermark supervision using the private carriers and key.
\emph{Quantization:} 4-bit symmetric uniform per-tensor fake quantization of matrix-valued weight parameters after training. The quantized values are stored in floating-point tensors; this experiment measures prediction changes, without claiming hardware low-bit speed or memory savings.

\subsection{Baselines}
\label{app:baselines}

We compare \textbf{InvGNN-WM} with:
\begin{itemize}
\item \textbf{SS} (\emph{task-only}): standard training without a watermark loss; a task-performance reference, not a guaranteed accuracy upper bound.
\item \textbf{COS}: a non-trigger adaptation that regresses invariant targets using cosine similarity between a pooled graph embedding and a learned anchor, rescaled to $[0,1]$.
\item \textbf{TRIG}~\citep{Zhao2021RG}: a trigger-style adaptation that trains task-class probabilities on carrier graphs to recover the owner bits.
\item \textbf{NAT}~\citep{Xu2023EuroSP} (our table label): a feature-trigger adaptation that writes $5w_k$ into feature channel zero of the first $\max\{1,\lfloor |V|/5\rfloor\}$ nodes of carrier $k$.
\item \textbf{EXPL}~\citep{Downer2025EXPL}: an attribution-style proxy that supplements scalar-target regression with a weight-$0.1$ loss aligning sigmoid-transformed, graph-pooled absolute input gradients with the invariant target.
\end{itemize}
These labels identify adaptations or proxies of the cited method families within the shared implementation, rather than full reproductions of their original algorithms. Comparisons share backbones, data splits, optimizer, and reporting protocol, with configured loss weights fixed before evaluation. The matched PreGIP-adapted comparison in Section~\ref{sec:pregip-comparison} uses separately calibrated method-specific detectors at the same target false-positive rate and query budget.

\paragraph{Common-edit outcomes and local applicability}
\label{app:rq1_tables}

\paragraph{Local applicability.}
Appendix~\ref{app:pl-sensitivity} reports the normalized-weight local-PL diagnostics, including checkpoints outside the diagnostic neighborhood. The reported constants characterize the evaluated setting; task accuracy is distinct from the task loss in Theorem~\ref{thm:impercept}. A loss bound is invoked only under all of the theorem's conditions.

\begin{table}[ht]
\centering
\caption{Common-edit results on PROTEINS--GIN (mean $\pm$95\% CI over three seeds). These are the InvGNN-WM results from the cross-method comparison in Table~\ref{tab:robust_baselines}. The KD row evaluates a callable scalar audit path; output-only extraction with no such endpoint is evaluated separately in Table~\ref{tab:interface_attacks}.}
\label{tab:robustness_budget}
\small
\begin{tabular}{lcc}
\toprule
\textbf{Edit} & \textbf{Task ACC (\%)} & \textbf{WM-ACC (\%)}\\
\midrule
None & $72.5\pm2.6$ & $89.8\pm0.0$\\
Pruning (40\%) & $71.6\pm2.8$ & $92.4\pm1.1$\\
Fine-tuning (20 epochs) & $73.1\pm2.5$ & $93.8\pm3.4$\\
KD (temperature $2$) & $72.0\pm2.7$ & $81.3\pm0.0$\\
Quantization (4-bit) & $72.1\pm2.7$ & $90.4\pm3.0$\\
\bottomrule
\end{tabular}
\end{table}

\noindent
All four edited mean WM-ACC values exceed the default acceptance reference $\tau/m=94/128\approx73.4\%$. KD produces the largest decrease from the clean mean, while the corresponding Task ACC remains within $0.5$ percentage points of its clean value. The recorded integer counts also show that all three seeds pass each of the four edits. The decision-preservation result requires the maximum score drift and the minimum clean margin measured on the same checkpoint and carrier set; aggregate measurements alone do not verify this premise.

\paragraph{Targeted head-parameter intervention.}
Table~\ref{tab:interface_attacks} distinguishes literal endpoint deletion from parameter reset and head-only randomization. Deletion makes the scalar audit unavailable, whereas the other interventions retain a callable path whose watermark accuracy can be measured.

\section{Extended diagnostics and analyses}
\label{app:extended}

\subsection{Sensitivity to invariant perturbations}
\label{app:inv_sensitivity}

\begin{figure}[t]
    \centering
    \includegraphics[width=0.62\linewidth]{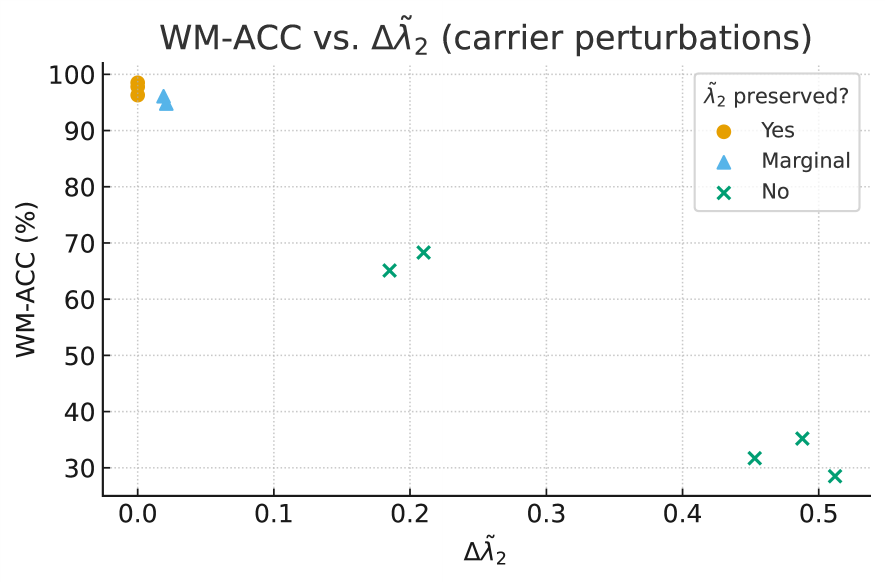}

    \caption{\textbf{WM-ACC vs.\ invariant perturbation.}
    Carriers are perturbed with increasing $\Delta\tilde{\lambda}_2$; marker colors and shapes indicate whether the invariant is (i) preserved, (ii) marginal, or (iii) broken.
    \emph{Observation.} WM-ACC is high when $\tilde{\lambda}_2$ is preserved and decreases under larger invariant perturbations; the strongest plotted perturbations yield WM-ACC below $50\%$.
    \emph{Scope.} The plot describes the evaluated perturbations to carrier graphs; it does not establish robustness to arbitrary parameter edits or endpoint deletion.}
    \label{fig:inv_sensitivity}

\end{figure}

\noindent\textbf{Analysis.}
Within this carrier-perturbation study, larger changes in normalized algebraic connectivity are associated with lower WM-ACC. A bound on invariant change alone is not a bound on the model's score drift. Parameter edits such as pruning leave a fixed input graph's algebraic connectivity unchanged but can still change its predicted score. Theorem~\ref{thm:robust} applies only when the actual score drift on the fixed carriers is below their clean decision margin.

\subsection{Score-feedback search proxy vs.\ nominal budget}
\label{app:forger_curves}

\begin{figure}[t]
    \centering
    \includegraphics[width=0.62\linewidth]{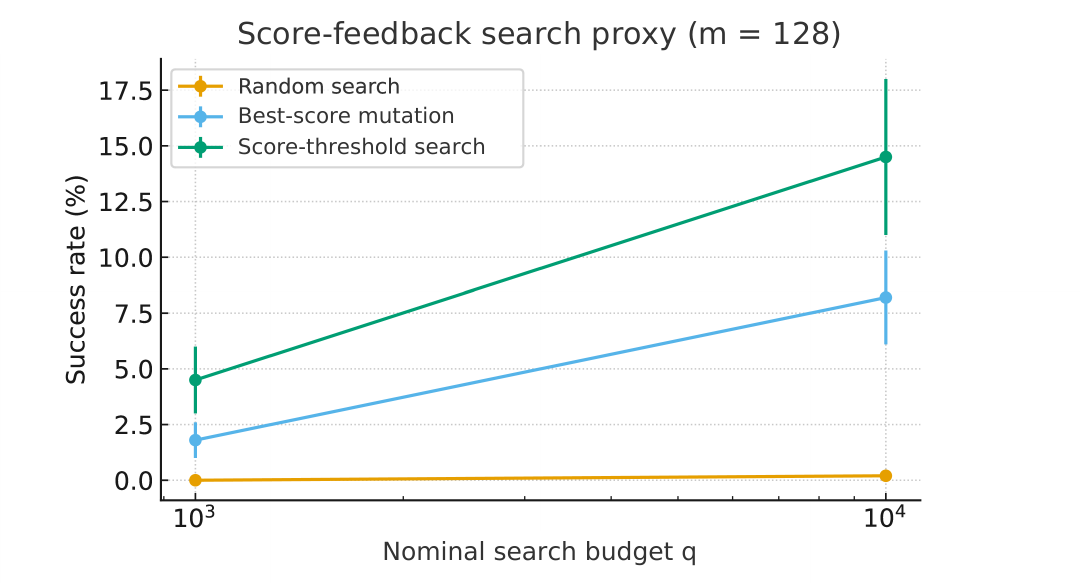}

    \caption{\textbf{Score-feedback search proxy} ($m{=}128$, target $\alpha{=}10^{-6}$).
    The strategies are random search, best-score mutation, and score-threshold search. Whiskers show the reported $95\%$ hierarchical-bootstrap intervals for repeated attempts.
    \emph{Observation.} Across the two plotted nominal budgets, success rates are higher at the larger budget, and the score-guided heuristic has the highest reported rate at both.
    \emph{Scope.} These outcomes describe a proxy with privileged owner-carrier score feedback; the horizontal axis is a search-budget parameter, not a count of fresh black-box model calls.}
    \label{fig:forger_curves}

\end{figure}

\noindent\textbf{Proxy protocol.}
A fixed model is evaluated once on the actual owner carriers to obtain a score vector. The search receives this vector, which is privileged feedback beyond the private-carrier threat model. Random search samples candidate bitstrings; the mutation heuristic selects the population member with highest agreement with the thresholded scores and flips its bits independently with probability $0.08$; score-guided search thresholds the scores at a randomly chosen quantile. For nominal budget $q$, each attempt generates $\max\{1,\lfloor q/m\rfloor\}$ proposals offline, with no additional model calls. Only the final proposal is evaluated against the private owner key, using match threshold $\tau=94$. Thus the plotted budgets $q=10^3,10^4$ index search effort after score acquisition. The outcomes characterize this surrogate procedure and do not estimate practical black-box forgery resistance or the Type-I error of a pre-committed null test.

\section{Limitations and Responsible Use}
\label{app:limitations}

\paragraph{Verification interfaces and attacker access.}
The auxiliary-head deployment requires a query-only scalar audit endpoint: the verifier obtains $s_\theta(G)$ without weights, gradients, training data, or hidden representations. Removing or disabling this endpoint prevents the audit, and head reset or longer retraining can reduce the watermark match count below threshold without a large task-accuracy loss. The additional task-logit deployment uses ordinary logits and a fixed owner-side decoder; it has no dedicated watermark endpoint, but still requires more than hard labels. Task-logit verification is evaluated only for graph classification. The attacker is gray-box: the architecture and procedure are known, while the owner's carriers and key remain private. Verification access and attacker knowledge are separate assumptions.

\paragraph{Fixed logit projection and output reparameterization.}
\label{app:logit-projection}
For a carrier target \(y=\tilde\lambda_2(G)\), returned logits \(z\in\mathbb R^C\),
and score \(s=d(z)=\sigma(a^\top z+b)\), differentiation gives
\[
\nabla_z(s-y)^2=2(s-y)s(1-s)a.
\]
The watermark loss thus supervises the fixed projection in direction \(a\); the gradient
can reach the task head and shared backbone through the logit map. The gradient formula
alone does not show that task predictions depend on the projected information.

Consider a softmax task API with at least two logits. For a scalar \(c\in\mathbb R\)
and the all-ones vector \(\mathbf1_C\), let \(z'=z+c\mathbf1_C\). The task probabilities
are unchanged, but the audit score satisfies
\[
\operatorname{softmax}(z')=\operatorname{softmax}(z),\qquad
d(z')=\sigma\!\left(a^\top z+b+c\,a^\top\mathbf1_C\right).
\]
When \(a^\top\mathbf1_C\ne0\), the audit score can therefore change without changing
task probabilities or predicted labels. For any fixed finite carrier set, a sufficiently
large common shift of either sign makes the decoded bits constant. Such a constant string
fails an audit if both the number of zero-bits and the number of one-bits in the owner key
are below \(\tau\). This is a structural observation, not a reported attack experiment.
The specified decoder uses \(a=\operatorname{linspace}(-1,1,C)\) and \(b=0\).
For \(C>1\), its coefficients sum to zero, so \(d(z+c\mathbf1_C)=d(z)\) exactly.
For a single-logit task, \(a=-1\); shifting the logit changes the task sigmoid probability,
so the preceding utility-preserving softmax argument does not apply.
These algebraic properties concern common shifts; other output transformations remain separate attack settings.

\paragraph{Conditional guarantees.}
The task-loss bound requires a differentiable stationary point inside the stated local Polyak--\L{}ojasiewicz neighborhood and the parameter-Lipschitz condition on that neighborhood. The chosen watermark weight must satisfy the stated bound. Empirical fits do not prove uniform or global regularity. The zero-bit-change result requires maximum carrier-score drift strictly below the minimum clean carrier margin. The partial-margin extension in Appendix~\ref{app:partial-margin} permits some changes when the remaining matches still reach \(\tau\); both results require the stated per-carrier conditions and an available interface. The exact-removal theorem concerns worst-case sparse removal of every bit within a rationally encoded, monotone separable decoder class. It does not establish hardness of practical evasion, arbitrary nonlinear checkpoints, or the task-logit deployment.

\paragraph{Null calibration, key separation, and repeated tests.}
False-positive control requires the stated pooled-null mean and joint concentration condition for one pre-specified detector. Small empirical pairwise correlations do not prove the joint condition. Key separation additionally requires independent carrier-set draws, sufficient balance, and bounded decoding error. The finite statistic bank and null trials are diagnostics for the tested protocol. Zero observed events at a given Monte Carlo resolution are not zero population risk or an upper confidence bound. Searches over carriers, keys, thresholds, checkpoints, or decoders, and uncorrected repeated testing, fall outside the single-test guarantee. Multi-key rules and any joint error-control procedure must be fixed before observing the suspect model.

\paragraph{Finite-budget and external-validation limits.}
The experiments cover the reported datasets, splits, seeds, backbones, attacks, and budgets. Ordinary KD can weaken or remove the watermark; KD+WM is owner-side recovery or transfer and requires private carriers and key. The task-logit variant has not been evaluated against the same pseudo-carrier-aware optimizer used on the scalar endpoint. Exact-carrier-aware anti-watermark co-training, alternative-invariant spoofing, collusion across stolen models, unlimited retraining, and denial of the complete task API are unvalidated. An additional OGB molecular benchmark and GPS backbone extend coverage but do not establish performance on heterogeneous, temporal, knowledge, higher-order, industrial, or production-API graph settings. The work uses public graph benchmarks, involves no human-subject study, and makes no human-population representativeness claim.

\paragraph{Invariant and computation scope.}
Algebraic connectivity is a useful target for the evaluated rewiring protocol, not a universal semantic invariant. Degree-preserving swaps preserve degrees but can change triangle counts and other structural properties. WL non-collision and KS checks are finite-sample filters, not proofs of population-support separation or universal indistinguishability. Small carriers reduce spectral preprocessing cost. Appendix~\ref{app:scale-compute} distinguishes per-run training times, peak memory, and audit latency from the resource subtotal and the larger independent-training detection study.

\paragraph{Intended use and misuse.}
The intended user is a model owner or explicitly authorized party who embeds a watermark before release and conducts a provenance audit. Legitimate uses include authorized deployment or transfer, license-compliance audits, and independent provenance review. Misuse includes unauthorized insertion into another party's model; retroactive key construction; post hoc search over verification material after inspecting a suspect model; selective reporting; and unsupported or coercive claims. A passing test means that responses match committed verification material above the calibrated threshold. A pass alone does not establish legal title, copying, infringement, a license breach, or whether use is authorized, protected by fair use, or otherwise lawful. A failed test does not establish absence of ownership. Responsible deployment requires secure verification material, pre-release commitments, a fixed audit procedure, and independent assessment of competing claims as described in Appendix~\ref{app:ownership-ambiguity}.

\section{Additional Ablations and Protocol Studies}
\label{app:additional_ablations}

\subsection{Carrier count and threshold}
\label{app:carrier_count}

\begin{table}[ht]
\centering
\caption{Carrier-count diagnostic on PROTEINS--GIN. Owner match counts are means with two-sided $95\%$ Student $t$ CIs across seeds $41/42/43$. Means, CI half-widths, and gaps are rounded separately to integers. Thresholds and gaps use the stated rule at a common $\alpha=10^{-6}$ and $\widehat c_{\mathrm{dep}}=4(7.6\times10^{-4})$.}
\label{tab:m_vs_tau}
\small
\begin{tabular}{ccccc}
\toprule
$m$ & $\varepsilon_{\mathrm{err}}$ & $\tau$ & Match count $T$ & $T-\tau$\\
\midrule
64  & 0.171 & 54  & $59\pm4$  & $+5$\\
96  & 0.231 & 74  & $88\pm3$  & $+14$\\
128 & 0.267 & 94  & $115\pm3$ & $+21$\\
192 & 0.310 & 133 & $174\pm2$ & $+41$\\
\bottomrule
\end{tabular}
\end{table}

At a fixed target false-positive rate, the normalized null threshold approaches $1/2$ as $m$ grows. All three observed seed counts exceed the corresponding threshold at each $m$; this observation does not establish a population-level success probability. Changing the carrier set requires an appropriate embedding and calibration protocol. The $m=128$ row uses counts $(114,114,116)$; the clean benchmark in Table~\ref{tab:uniq_strip} uses $(115,115,115)$. These are distinct count records.

\subsection{Carrier-generation thresholds}
\label{app:protocol_ablation}

\begin{table}

\centering
\caption{Carrier generation thresholds (PROTEINS--GIN).}
\label{tab:protocol_thresholds}
\scriptsize
\setlength{\tabcolsep}{4pt}
\renewcommand{\arraystretch}{0.95}
\begin{tabular}{cccccc}
\toprule
Swap cap & KS $\delta$ & Task ACC (\%) & WM-ACC (\%) & $\hat\rho_0$ & Measured $\alpha$ ($10^7$ trials) \\
\midrule
5   & 0.05 & 72.6 $\pm$ 2.6 & 88.0 $\pm$ 12.3 & $9.1\times 10^{-4}$ & $<10^{-6}$ \\
25  & 0.10 & 72.5 $\pm$ 2.6 & 89.6 $\pm$ 1.1 & $8.2\times 10^{-4}$ & $<10^{-7}$ \\
50  & 0.10 & 72.5 $\pm$ 2.6 & 89.8 $\pm$ 0.0 & $7.6\times 10^{-4}$ & $<10^{-7}$ \\
50  & 0.20 & 72.4 $\pm$ 2.7 & 89.1 $\pm$ 0.0 & $7.1\times 10^{-4}$ & $<10^{-7}$ \\
\bottomrule
\end{tabular}

\end{table}

The default swap cap and KS setting give high WM-ACC and a small reported dependence diagnostic $\hat\rho_0$. The diagnostic $\hat\rho_0$ is not a false-positive rate.
The empirical null rates are reported at the resolution of $10^7$ trials. A label $<10^{-7}$ indicates zero observed events at that resolution, not a population upper confidence bound. For zero events in $10^7$ independent trials, the one-sided $95\%$ Clopper--Pearson upper bound is approximately $3\times10^{-7}$; the two-sided $95\%$ interval has upper endpoint approximately $3.6889\times10^{-7}$.

\noindent\textbf{Takeaway.}
The default sampler (swap cap $50$, KS $\delta{=}0.10$) gives high WM-ACC and zero observed null events at the stated resolution in this setting.

\section{Approximate eigen-solvers for \texorpdfstring{$\lambda_2$}{lambda2}}
\label{app:approx-eigs}

For very large graphs, an exact eigen-decomposition of the Laplacian can be replaced by
iterative solvers such as Lanczos or an appropriately shifted and constrained iteration.
To assess the effect on watermark performance, we repeat the PROTEINS--GIN experiments
with $I(G)$ defined as the output of the reported $K$-step iterative approximation targeting the second
smallest eigenvalue.
We fix $K=10$ and keep all other settings unchanged.

\begin{table}[t]
\centering
\caption{Exact vs.\ approximate computation of $\lambda_2$ on PROTEINS--GIN.
All runs use the same carriers and backbone.}
\label{tab:approx_lambda2}
\scriptsize
\setlength{\tabcolsep}{6pt}
\renewcommand{\arraystretch}{0.95}
\begin{tabular}{lcccc}
\toprule
\textbf{Method} & \textbf{Task ACC (\%)} & \textbf{WM-ACC (\%)} &
$\kappa_{\mathrm{marg}}$ & \textbf{Time / carrier (ms)} \\
\midrule
Exact eigensolver       & $72.5 \pm 2.6$ & $89.8 \pm 0.0$ &
$0.382$ & $0.08$ \\
Iterative approximation ($K=10$)& $72.4 \pm 2.7$ & $88.8 \pm 1.1$ &
$0.371$ & $0.05$ \\
\bottomrule
\end{tabular}
\end{table}

The reported mean Task ACC is $72.5\%$ with the exact eigensolver and $72.4\%$
with the iterative approximation. Mean WM-ACC differs by approximately
$1.0$ percentage point between the two computations of $\lambda_2$.
The robustness margin $\kappa_{\mathrm{marg}}$ is also stable, while the per-carrier
computation time is reduced from $t_{\text{exact}} = 0.08$\,ms to
$t_{\text{power}} = 0.05$\,ms.
The experiment supports using an iterative approximation on these carriers. It does not benchmark the solver on industrial-scale graphs.

\section{Additional diagnostics linking theory and experiments}
\label{app:theory_experiment_plots}

The local applicability study in Appendix~\ref{app:pl-sensitivity} uses the normalized watermark weight $\bar\beta=\beta_{\mathrm{wm}}/\beta_{\max}$, the diagnostic neighborhood ratio, and observed task-loss changes. Plug-in reference values are displayed for settings within that neighborhood and with $\bar\beta\leq1$; the theorem additionally requires stationarity and uniform local regularity.

Common-edit WM-ACC is summarized in Table~\ref{tab:robustness_budget} and compared across methods in Table~\ref{tab:robust_baselines}. A valid use of the deterministic margin result requires paired measurements of the maximum score drift and minimum margin on each checkpoint and its fixed carrier set. Seed-averaged WM-ACC and separately aggregated drift summaries cannot certify unchanged decoded bits. We therefore use these measurements as empirical outcomes rather than a verification of every premise of Theorem~\ref{thm:robust}.

\section{Additional robustness and adaptive-attack experiments}
\label{app:robustness-adaptive}

This appendix adds robustness diagnostics that complement the main results in
Section~\ref{sec:experiments}.
We first extend the robustness comparison to baseline watermarking schemes,
then study a watermark-aware adaptive attack, and finally test whether an
attacker can distinguish carrier graphs from task graphs.

\subsection{Extended robustness comparison with baseline methods}
\label{app:robust-baselines}

Table~\ref{tab:robust_baselines} reports test \textbf{Task ACC} and
\textbf{WM-ACC} for two representative dataset--backbone pairs
(Cora--GCN and PROTEINS--GIN) under several common edits.
For each setting we evaluate three strong watermarking baselines
(TRIG, NAT, EXPL) and InvGNN-WM with the same backbone and training
hyperparameters as in Table~\ref{tab:main_comparison_compact}.
Edits include unstructured pruning, short fine-tuning on task data,
plain KD, and post-training 4-bit quantization.
This table makes it possible to compare, side by side, how much
utility and watermark accuracy each method retains after aggressive
model changes.

\begin{table}[t]
\centering
\caption{\textbf{Robustness of watermarking methods under common edits.}
For each dataset--backbone and edit, two rows report test Task ACC (\%) and carrier
WM-ACC (\%), abbreviated as \textbf{W-A} (mean $\pm$95\% CI over three seeds).
The InvGNN-WM results on PROTEINS--GIN are also summarized in Table~\ref{tab:robustness_budget}.}
\label{tab:robust_baselines}
\scriptsize
\setlength{\tabcolsep}{4.0pt}
\renewcommand{\arraystretch}{0.92}
\resizebox{\textwidth}{!}{%
\begin{tabular}{llccccc}
\toprule
\textbf{Dataset--Backbone} & \textbf{Edit} & \textbf{Metric} &
\textbf{TRIG} & \textbf{NAT} & \textbf{EXPL} & \textbf{InvGNN-WM} \\
\midrule
\multirow{10}{*}{Cora--GCN}
& \multirow{2}{*}{None}
& Task & $86.8 \pm 0.9$ & $86.9 \pm 0.9$ & $86.7 \pm 1.0$ & $87.0 \pm 0.8$ \\
& & W-A & $97.7 \pm 3.4$ & $96.4 \pm 1.1$ & $93.0 \pm 0.0$ & $99.0 \pm 1.1$ \\
\cmidrule(lr){2-7}
& \multirow{2}{*}{Pruning 40\%}
& Task & $85.9 \pm 1.2$ & $86.0 \pm 1.2$ & $85.5 \pm 1.3$ & $86.4 \pm 1.0$ \\
& & W-A & $93.8 \pm 0.0$ & $92.7 \pm 1.1$ & $88.3 \pm 0.0$ & $97.1 \pm 3.0$ \\
\cmidrule(lr){2-7}
& \multirow{2}{*}{Fine-tuning}
& Task & $87.1 \pm 0.9$ & $87.2 \pm 0.9$ & $86.9 \pm 1.1$ & $87.3 \pm 0.8$ \\
& & W-A & $96.9 \pm 0.0$ & $95.3 \pm 0.0$ & $91.9 \pm 6.2$ & $98.2 \pm 1.1$ \\
\cmidrule(lr){2-7}
& \multirow{2}{*}{KD ($T{=}2$)}
& Task & $86.3 \pm 1.0$ & $86.1 \pm 1.1$ & $85.8 \pm 1.2$ & $86.5 \pm 0.9$ \\
& & W-A & $84.4 \pm 15.9$ & $81.8 \pm 10.7$ & $72.9 \pm 1.1$ & $90.1 \pm 1.1$ \\
\cmidrule(lr){2-7}
& \multirow{2}{*}{PTQ 4-bit}
& Task & $86.0 \pm 1.1$ & $86.0 \pm 1.1$ & $85.4 \pm 1.3$ & $86.3 \pm 0.9$ \\
& & W-A & $92.4 \pm 1.1$ & $90.9 \pm 1.1$ & $86.2 \pm 1.1$ & $96.1 \pm 0.0$ \\
\midrule
\multirow{10}{*}{PROTEINS--GIN}
& \multirow{2}{*}{None}
& Task & $72.8 \pm 2.6$ & $72.6 \pm 2.7$ & $72.4 \pm 2.8$ & $72.5 \pm 2.6$ \\
& & W-A & $95.1 \pm 3.0$ & $94.8 \pm 1.1$ & $90.6 \pm 0.0$ & $89.8 \pm 0.0$ \\
\cmidrule(lr){2-7}
& \multirow{2}{*}{Pruning 40\%}
& Task & $71.0 \pm 2.9$ & $70.9 \pm 3.0$ & $70.5 \pm 3.1$ & $71.6 \pm 2.8$ \\
& & W-A & $90.1 \pm 1.1$ & $89.8 \pm 0.0$ & $84.4 \pm 15.9$ & $92.4 \pm 1.1$ \\
\cmidrule(lr){2-7}
& \multirow{2}{*}{Fine-tuning}
& Task & $73.0 \pm 2.5$ & $72.9 \pm 2.6$ & $72.6 \pm 2.7$ & $73.1 \pm 2.5$ \\
& & W-A & $93.8 \pm 0.0$ & $93.2 \pm 3.0$ & $88.8 \pm 1.1$ & $93.8 \pm 3.4$ \\
\cmidrule(lr){2-7}
& \multirow{2}{*}{KD ($T{=}2$)}
& Task & $71.8 \pm 2.7$ & $71.5 \pm 2.8$ & $71.1 \pm 2.9$ & $72.0 \pm 2.7$ \\
& & W-A & $78.6 \pm 1.1$ & $76.8 \pm 10.7$ & $69.5 \pm 0.0$ & $81.3 \pm 0.0$ \\
\cmidrule(lr){2-7}
& \multirow{2}{*}{PTQ 4-bit}
& Task & $71.5 \pm 2.8$ & $71.3 \pm 2.9$ & $70.8 \pm 3.0$ & $72.1 \pm 2.7$ \\
& & W-A & $88.8 \pm 1.1$ & $87.8 \pm 1.1$ & $83.1 \pm 12.3$ & $90.4 \pm 3.0$ \\
\bottomrule
\end{tabular}%
}
\end{table}

Within Table~\ref{tab:robust_baselines}, InvGNN-WM has the highest mean WM-ACC after
pruning, KD, and 4-bit quantization on both datasets. On PROTEINS--GIN, TRIG has higher
mean WM-ACC before editing; the two methods have equal mean WM-ACC after fine-tuning.
On PROTEINS--GIN, the mean WM-ACC advantages over the strongest baseline are
$2.3$, $2.6$, and $1.6$ percentage points under pruning, KD, and 4-bit quantization,
respectively, computed before rounding the individual means. These are descriptive mean
differences, not claims of statistical significance. Raw WM-ACC comparisons are distinct from
matched detection at a common false-positive target in Section~\ref{sec:pregip-comparison}.

\subsection{Watermark-aware adaptive attack}
\label{app:adaptive-attack}

We consider an attacker who knows the watermarking
mechanism but does not know the actual carrier set $\mathcal{G}_W$.
The attacker knows that the scalar perception head $s_\theta(G)$ is
trained to regress a graph invariant $I(G)$ and tries to remove
invariant-dependent variation by moving its output toward $0.5$ on
carrier-like graphs.

Concretely, for PROTEINS--GIN the attacker:
(i) samples a pool of \emph{pseudo-carriers} by running the same
degree-preserving rewiring procedure as the owner with a different
random seed and without the private owner key;
(ii) pairs task minibatches with pseudo-carrier minibatches of size 8;
and (iii) fine-tunes the stolen model for 100 epochs using
$\mathcal L_{\mathrm{task}}+0.5\,\mathbb E_{G\sim\mathcal G_{\mathrm{pseudo}}}
[(s_\theta(G)-0.5)^2]$.
Adam uses learning rate $0.002$ and weight decay $5\times10^{-4}$.
The squared-error term collapses scores toward the decoding boundary,
while the task term preserves task utility during optimization.

Table~\ref{tab:adaptive_attack} summarizes the outcome of this
adaptive attack for TRIG, NAT, EXPL, and InvGNN-WM.
The table reports test Task ACC and WM-ACC before and after the
attack, allowing a direct comparison of how robust each method is
against this stronger threat model.

\begin{table}[t]
\centering
\caption{Watermark-aware adaptive attack on PROTEINS--GIN.
For each watermarking method we report test Task ACC (\%) and WM-ACC
(\%) before and after the adaptive attack.}
\label{tab:adaptive_attack}
\scriptsize
\setlength{\tabcolsep}{6pt}
\renewcommand{\arraystretch}{0.95}
\begin{tabular}{lcccc}
\toprule
\multirow{2}{*}{Method} &
\multicolumn{2}{c}{Before adaptive attack} &
\multicolumn{2}{c}{After adaptive attack} \\
\cmidrule(lr){2-3}\cmidrule(lr){4-5}
& Task ACC & WM-ACC & Task ACC & WM-ACC \\
\midrule
TRIG        & $72.8 \pm 2.6$ & $95.1 \pm 3.0$ &
              $72.0 \pm 2.7$ & $61.5 \pm 1.1$ \\
NAT         & $72.6 \pm 2.7$ & $94.8 \pm 1.1$ &
              $71.9 \pm 2.8$ & $58.6 \pm 0.0$ \\
EXPL        & $72.4 \pm 2.8$ & $90.6 \pm 0.0$ &
              $71.7 \pm 2.9$ & $55.2 \pm 1.1$ \\
InvGNN-WM   & $72.5 \pm 2.6$ & $89.8 \pm 0.0$ &
              $72.1 \pm 2.7$ & $78.4 \pm 1.1$ \\
\bottomrule
\end{tabular}
\end{table}

The adaptive attack substantially reduces WM-ACC for all methods, but
InvGNN-WM retains a higher mean watermark accuracy while
keeping Task ACC essentially unchanged.
In contrast, TRIG, NAT, and EXPL experience large drops in WM-ACC that
bring them closer to the random-guessing regime on carriers. For InvGNN-WM, the three match counts are $100$, $100$, and $101$, so all three seeds remain verified at $\tau=94$, with mean WM-ACC $78.4\%$. This attack was evaluated on the scalar endpoint, not on the task-logit variant.

\subsection{Carrier indistinguishability experiment}
\label{app:carrier-secrecy}

Finally, we examine whether an attacker can reliably tell apart carrier
graphs from ordinary task graphs.
For PROTEINS--GIN we construct a balanced binary classification task
where the positive class contains carrier graphs from $\mathcal{G}_W$
and the negative class contains an equal number of task graphs drawn
from the task training pool without replacement.
We use stratified train/validation/test splits with
ratios $80/10/10$ and a four-layer GCN with hidden width 64, trained for 100 epochs using
Adam and the best validation checkpoint, as in Section~\ref{sec:exp_setup}.
The target accuracy for random guessing on this task is $50\%$.

Table~\ref{tab:carrier_detection} reports the classification accuracy
on the test split, together with precision and recall for the carrier
class.
The result stays close to random guessing and does not reveal a stable
decision boundary between carrier and non-carrier graphs, which
provides a diagnostic for this classifier and split. Near-chance accuracy for one classifier does not prove indistinguishability against every attacker.

\begin{table}[t]
\centering
\caption{Carrier vs.\ non-carrier detection on PROTEINS--GIN.
The table reports test accuracy (ACC), precision (Prec), and recall
(Rec) for the carrier class on the binary classification task
described in Section~\ref{app:carrier-secrecy}.}
\label{tab:carrier_detection}
\scriptsize
\setlength{\tabcolsep}{8pt}
\renewcommand{\arraystretch}{0.95}
\begin{tabular}{lccc}
\toprule
Model & ACC (\%) & Prec (\%) & Rec (\%) \\
\midrule
GCN classifier & $51.3$ & $52.1$ & $50.2$ \\
\bottomrule
\end{tabular}
\end{table}

\section{Competing Claims and Threshold Sensitivity}
\label{app:ownership-ambiguity}

A watermark test measures agreement with a fixed key on fixed private carriers. Key-separation results concern independently generated keys under their stated assumptions; they do not settle disputes about how verification material was created or who lawfully owns a model.

\subsection{Independent keys and competing claims}
\label{app:ownership-ambiguity-scenario}

Suppose parties A and B use the same public architecture, invariant, and carrier-generation procedure but generate independent private carrier sets and keys. Theorem~\ref{thm:uniq} describes separation of the resulting key-conditioned decoded-bit distributions under balance, concentration, and decoding-error assumptions. A rival's test on the suspect model must separately satisfy its own null assumptions. Key separation alone does not prove that the suspect behaves as a null model on every rival carrier set, especially after adaptive key construction or training with multiple watermarks.

Multiple keys may reduce accidental conflicts only under a pre-specified combination rule and valid joint error control. Each subtest must be calibrated for its own key length. For example, a threshold derived for 128 bits cannot be applied to a 64-bit subtest. Multiplication of false-positive rates requires the relevant joint independence conditions; it is not implied by merely using disjoint carrier sets. Multi-key testing does not replace the dispute procedure below.

\subsection{Pre-release commitment and independent audit}
\label{app:dispute-procedure}

\paragraph{Before release.}
Each claimant should publish a binding, hiding, independently timestamped cryptographic commitment to a canonical encoding of the carrier set, key, decoder, verification threshold, null-calibration procedure, reference-checkpoint digest, and any multi-key combination rule. A high-entropy nonce should be included in the commitment. The detector, any multiple-testing correction, and the rule for identifying the model under audit must be fixed before observing a suspect model. The commitment must protect the secrecy of the carriers and key while allowing later verification of the exact committed material.

\paragraph{During a dispute.}
Each claimant reveals the committed material to an independent auditor. The auditor checks that the material matches the unchanged pre-release commitment, records the exact suspect-checkpoint digest, and runs each pre-specified detector with its original threshold. The audit does not allow retuning, retrospective key search, or selective omission of failed tests. A claim receives technical support only when its prior commitment is valid and its specified test passes. A failed or unavailable test means that the watermark evidence does not support that claim; it does not prove absence of ownership.

\paragraph{Reporting the outcome.}
If exactly one claim receives technical support, the auditor may report that the watermark evidence favors that claimant, while leaving legal ownership undetermined. The outcome is inconclusive, with no ownership assignment, when no claim or multiple claims receive technical support, a valid prior commitment is missing, the checkpoint or model lineage is disputed, or watermark results conflict with external provenance evidence. An earlier timestamp alone is insufficient. The auditor should consider training records, model and code signatures, licenses, release timestamps, and other chain-of-custody records together with the watermark evidence. The procedure is an evidence-handling recommendation, not a cryptographic proof of authorship or a legal adjudication rule.

\subsection{Threshold sweep: trade-off between \texorpdfstring{$\alpha$}{alpha} and false negatives}
\label{app:tau-sweep}

We now quantify how the watermark test behaves when the verification
threshold is slightly relaxed or tightened.
For PROTEINS--GIN we sweep a family of thresholds
$\tau' \in \{\tau^\ast\!-\!2,\tau^\ast\!-\!1,\tau^\ast,\tau^\ast\!+\!1,
\tau^\ast\!+\!2\}$ around the nominal value $\tau^\ast{=}94$ used in
the main experiments.
For each $\tau'$ we report the null false-positive estimate $\alpha(\tau')$
and the false-negative estimate $\beta_{\mathrm{fn}}(\tau')$ on the positive checkpoint pool.
This pool comprises PROTEINS--GIN auxiliary-head checkpoints in the Clean condition
and after 20 epochs of task-only fine-tuning (FT20), evaluated with their matching owner keys and carriers.

\begin{table}[t]
\centering
\caption{Threshold-sweep point estimates for the PROTEINS--GIN auxiliary-head detector. The positive checkpoint pool comprises Clean and FT20 conditions. The entries are diagnostic estimates, not confidence intervals or condition-specific success guarantees.}
\label{tab:tau_sweep}
\scriptsize
\setlength{\tabcolsep}{8pt}
\renewcommand{\arraystretch}{0.95}
\begin{tabular}{ccc}
\toprule
$\tau'$ & $\alpha(\tau')$ & $\beta_{\mathrm{fn}}(\tau')$ \\
\midrule
$\tau^\ast - 2 = 92$ & $7.8\times 10^{-5}$ & $0.000$ \\
$\tau^\ast - 1 = 93$ & $1.6\times 10^{-5}$ & $0.000$ \\
$\tau^\ast       = 94$ & $6.0\times 10^{-8}$ & $0.000$ \\
$\tau^\ast + 1 = 95$ & $3.0\times 10^{-8}$ & $3.5\times 10^{-3}$ \\
$\tau^\ast + 2 = 96$ & $1.0\times 10^{-8}$ & $1.1\times 10^{-2}$ \\
\bottomrule
\end{tabular}
\end{table}

Lower thresholds increase the reported null match rate; higher thresholds introduce more false negatives in the evaluated checkpoint pool. The nominal threshold is fixed before attack evaluation. The false-negative column summarizes the combined Clean/FT20 pool, rather than either condition separately. Values below $10^{-7}$ in Table~\ref{tab:tau_sweep} are diagnostic point estimates, distinct from both confidence bounds and the zero-event resolution labels in the $10^7$-trial null experiment.

\section{Additional details on algebraic connectivity}
\label{app:lambda-choice}

\paragraph{Why algebraic connectivity \texorpdfstring{$\lambda_2$}{lambda2}?}

Algebraic connectivity $\lambda_2(G)$ of graph $G$ is the second smallest eigenvalue of the graph Laplacian. By classical results, $\lambda_2(G)=0$ if and only if $G$ is disconnected, and Cheeger-type inequalities relate $\lambda_2(G)$ to edge expansion and sparsest cuts, so $\lambda_2(G)$ contains information about connectivity and cut structure. The edge swaps preserve degrees exactly, and the KS filter tests selected local and spectral statistics. Neither the KS filter nor degree preservation guarantees that all local structure is unchanged or that swaps affect only long-range connections. Degree moments are preserved by degree-preserving swaps, whereas triangle counts can change; the invariant comparison below measures the effect of each target under the chosen protocol. In contrast, $\lambda_2(G)$ changes in a stable and informative way when we weaken connectivity across coarse cuts, making it a natural signal for the perception head. Section~\ref{sec:inv-choice-exp} compares the normalized invariant $\tilde\lambda_2(G)$ with spectral radius and triangle count and shows that $\tilde\lambda_2(G)$ yields higher WM-ACC and a larger robustness margin $\kappa_{\mathrm{marg}}$ under the same protocol.

\paragraph{Complexity and scalability.}

Exact computation of $\lambda_2(G)$ for a carrier with $n$ nodes via a dense eigensolver costs $O(n^3)$ time. To keep this step cheap, every carrier satisfies $n \le n_{\mathrm{cap}}=\max\{4,\min\{32,\lfloor n_{0.25}\rfloor\}\}$, where $n_{0.25}$ is the $25$-th percentile of task-graph node counts (Appendix~\ref{sec:appendix-protocol-data}). The cap combines a four-node floor, a percentile-based limit, and an absolute limit of 32 nodes. On PROTEINS, spectral preprocessing takes about $0.08$\,ms per carrier, corresponding to about $10.24$\,ms for $m=128$ carriers. This preprocessing is not the model-query audit latency; the separate OGB audit timings are reported in Appendix~\ref{app:scale-compute}. For larger graphs, the same framework can use iterative solvers such as Lanczos or an appropriately shifted and constrained iteration with per-iteration cost $O(|V|+|E|)$. In Appendix~\ref{app:approx-eigs} we replace the exact eigensolver by the reported $K{=}10$-step iterative approximation and observe that WM-ACC changes by approximately $1.0$ percentage point, while the per-carrier time drops from $0.08$\,ms to $0.05$\,ms. These timings and accuracy differences describe the tested small carriers; they do not establish end-to-end latency for arbitrarily large graphs.

\subsection{Invariant choice and interpretability: why \texorpdfstring{$\lambda_2$}{lambda2}?}
\label{sec:inv-choice-exp}

We next ask how the choice of invariant $I(G)$ affects watermark strength and robustness.
We keep the carrier set, backbone (PROTEINS--GIN), and training protocol fixed, and only
change the scalar target that the perception head $s_\theta(G)$ is trained to regress.

\begin{table}[ht]

\centering
\caption{Invariant choice on PROTEINS--GIN (same carriers/backbone).
Each row reports Task ACC, WM-ACC, and robustness margin $\kappa_{\mathrm{marg}}$.}
\label{tab:inv_choice}
\scriptsize
\setlength{\tabcolsep}{4pt}
\renewcommand{\arraystretch}{0.95}
\begin{tabular}{lccc}
\toprule
\textbf{Invariant} & \textbf{Task ACC (\%)} & \textbf{WM-ACC (\%)} & $\kappa_{\mathrm{marg}}$ \\
\midrule
$\tilde{\lambda}_2$ (ours)        & $72.5 \pm 2.6$ & $89.8 \pm 0.0$ & $0.382$ \\
Spectral radius (normalized)      & $72.1 \pm 2.7$ & $87.5 \pm 0.0$ & $0.351$ \\
Triangle count (normalized)       & $71.9 \pm 2.8$ & $84.4 \pm 0.0$ & $0.315$ \\
\bottomrule
\end{tabular}

\end{table}

\noindent\textbf{Results.}
All three choices keep Task ACC within a narrow observed band in this experiment.
However, both WM-ACC and the robustness margin $\kappa_{\mathrm{marg}}$ are highest when
$I(G)=\tilde{\lambda}_2(G)$.
The normalized spectral-radius target gives a slightly smaller reported margin, and triangle count gives the lowest WM-ACC among these three targets.

\noindent\textbf{Interpretation.}
Algebraic connectivity gives the highest reported WM-ACC and margin among these three choices under the same degree-preserving rewiring protocol. The result supports the chosen target for this setting. It does not imply that triangle counts are invariant under rewiring, that every swap weakens global connectivity, or that task predictions causally depend on $\lambda_2$.

\section{Full main comparison across all dataset--backbone settings}
\label{app:full-main-comparison}

This appendix reports the complete 13 dataset--backbone grid (9 node-level and
4 graph-level settings) corresponding to the benchmark comparison in
Section~\ref{sec:comparative}. Table~\ref{tab:main_comparison_compact} reports representative backbones in Appendix~\ref{app:original-comparisons}.

\begin{table}[t]
\centering
\caption{\textbf{Full comparison (node-level datasets).}
Each cell shows \emph{test} \textbf{Task ACC} (\%) on the first line and
\textbf{WM-ACC on private carriers} (\%) on the second line
(mean $\pm$95\% CI; three seeds). Best Task ACC per \emph{row}
\emph{excluding} SS is in \textbf{bold}; the \bestwm{best WM-ACC} per
\emph{row} is teal bold.}
\label{tab:full_main_comparison_node}

\setlength{\tabcolsep}{2.5pt}
\renewcommand{\arraystretch}{1.03}
\small
\renewcommand{\accwm}[2]{\makecell{#1\\#2}}
\begin{tabular}{lcccccc}
\toprule
\textbf{Dataset--Backbone} & \textbf{SS} & \textbf{COS} & \textbf{TRIG} & \textbf{NAT} & \textbf{EXPL} & \textbf{OURS} \\
\midrule
\makecell[l]{Cora\\GCN}
& \accwm{87.2 $\pm$ 0.8}{49.5 $\pm$ 1.1}
& \accwm{85.5 $\pm$ 1.2}{86.2 $\pm$ 1.1}
& \accwm{86.8 $\pm$ 0.9}{97.7 $\pm$ 3.4}
& \accwm{86.9 $\pm$ 0.9}{96.4 $\pm$ 1.1}
& \accwm{86.7 $\pm$ 1.0}{93.0 $\pm$ 0.0}
& \ourscol \accwm{\bestacc{87.0 $\pm$ 0.8}}{\bestwm{99.0 $\pm$ 1.1}} \\
\makecell[l]{Cora\\GraphSAGE}
& \accwm{84.0 $\pm$ 1.0}{50.0 $\pm$ 0.0}
& \accwm{82.1 $\pm$ 1.5}{83.9 $\pm$ 1.1}
& \accwm{\bestacc{83.9 $\pm$ 1.1}}{97.1 $\pm$ 1.1}
& \accwm{83.8 $\pm$ 1.2}{96.9 $\pm$ 0.0}
& \accwm{83.7 $\pm$ 1.1}{92.4 $\pm$ 1.1}
& \ourscol \accwm{83.8 $\pm$ 1.0}{\bestwm{98.4 $\pm$ 3.4}} \\
\makecell[l]{Cora\\SGC}
& \accwm{87.0 $\pm$ 0.9}{51.3 $\pm$ 1.1}
& \accwm{85.2 $\pm$ 1.3}{85.9 $\pm$ 0.0}
& \accwm{86.5 $\pm$ 1.0}{96.9 $\pm$ 0.0}
& \accwm{86.6 $\pm$ 1.0}{96.4 $\pm$ 1.1}
& \accwm{\bestacc{86.7 $\pm$ 1.1}}{92.7 $\pm$ 1.1}
& \ourscol \accwm{86.2 $\pm$ 1.0}{\bestwm{98.7 $\pm$ 1.1}} \\
\midrule
\makecell[l]{PubMed\\GCN}
& \accwm{88.6 $\pm$ 0.9}{49.7 $\pm$ 4.5}
& \accwm{86.4 $\pm$ 1.4}{87.5 $\pm$ 0.0}
& \accwm{87.9 $\pm$ 1.0}{96.9 $\pm$ 0.0}
& \accwm{87.8 $\pm$ 1.1}{96.6 $\pm$ 1.1}
& \accwm{85.7 $\pm$ 1.5}{94.3 $\pm$ 10.7}
& \ourscol \accwm{\bestacc{88.1 $\pm$ 1.0}}{\bestwm{98.7 $\pm$ 1.1}} \\
\makecell[l]{PubMed\\GraphSAGE}
& \accwm{91.2 $\pm$ 0.8}{51.3 $\pm$ 1.1}
& \accwm{89.0 $\pm$ 1.0}{88.0 $\pm$ 12.3}
& \accwm{90.1 $\pm$ 0.8}{96.4 $\pm$ 1.1}
& \accwm{90.0 $\pm$ 0.9}{96.1 $\pm$ 0.0}
& \accwm{\bestacc{91.3 $\pm$ 0.9}}{94.0 $\pm$ 1.1}
& \ourscol \accwm{90.7 $\pm$ 0.8}{\bestwm{98.2 $\pm$ 1.1}} \\
\makecell[l]{PubMed\\SGC}
& \accwm{88.8 $\pm$ 0.9}{50.3 $\pm$ 1.1}
& \accwm{86.7 $\pm$ 1.3}{87.0 $\pm$ 1.1}
& \accwm{\bestacc{88.1 $\pm$ 1.0}}{96.9 $\pm$ 0.0}
& \accwm{88.0 $\pm$ 1.1}{96.4 $\pm$ 1.1}
& \accwm{85.3 $\pm$ 1.6}{93.8 $\pm$ 0.0}
& \ourscol \accwm{87.7 $\pm$ 1.1}{\bestwm{98.7 $\pm$ 1.1}} \\
\midrule
\makecell[l]{AmazonPhoto\\GCN}
& \accwm{91.3 $\pm$ 0.6}{49.2 $\pm$ 0.0}
& \accwm{89.5 $\pm$ 1.1}{88.3 $\pm$ 0.0}
& \accwm{90.8 $\pm$ 0.7}{97.9 $\pm$ 1.1}
& \accwm{90.7 $\pm$ 0.8}{97.4 $\pm$ 1.1}
& \accwm{90.9 $\pm$ 0.8}{94.8 $\pm$ 1.1}
& \ourscol \accwm{\bestacc{91.1 $\pm$ 0.6}}{\bestwm{99.2 $\pm$ 0.0}} \\
\makecell[l]{AmazonPhoto\\GraphSAGE}
& \accwm{94.2 $\pm$ 0.5}{50.8 $\pm$ 0.0}
& \accwm{92.1 $\pm$ 1.0}{89.1 $\pm$ 0.0}
& \accwm{93.8 $\pm$ 0.6}{97.9 $\pm$ 1.1}
& \accwm{93.7 $\pm$ 0.6}{97.7 $\pm$ 0.0}
& \accwm{93.9 $\pm$ 0.7}{95.3 $\pm$ 0.0}
& \ourscol \accwm{\bestacc{94.0 $\pm$ 0.5}}{\bestwm{99.2 $\pm$ 0.0}} \\
\makecell[l]{AmazonPhoto\\SGC}
& \accwm{91.4 $\pm$ 0.6}{49.0 $\pm$ 1.1}
& \accwm{89.6 $\pm$ 1.2}{88.0 $\pm$ 1.1}
& \accwm{90.9 $\pm$ 0.7}{97.7 $\pm$ 0.0}
& \accwm{90.8 $\pm$ 0.8}{97.4 $\pm$ 1.1}
& \accwm{90.1 $\pm$ 0.9}{94.5 $\pm$ 0.0}
& \ourscol \accwm{\bestacc{91.0 $\pm$ 0.7}}{\bestwm{99.0 $\pm$ 3.0}} \\
\bottomrule
\end{tabular}
\end{table}

\begin{table*}[t]
\centering
\caption{\textbf{Full comparison (graph-level datasets; transposed).}
Each cell shows \emph{test} \textbf{Task ACC} (\%) on the first line and
\textbf{WM-ACC on private carriers} (\%), abbreviated as \textbf{W-A}, on the second
(mean $\pm$95\% CI; three seeds). Best Task ACC per \emph{column}
\emph{excluding} SS is in \textbf{bold}; the \bestwm{best WM-ACC} per
\emph{column} is teal bold.}
\label{tab:full_main_comparison_graph}

\setlength{\tabcolsep}{3.0pt}
\renewcommand{\arraystretch}{1.03}
\scriptsize

\begin{tabular}{lcccc}
\toprule
\textbf{Method}
& \textbf{PROTEINS--GIN}
& \textbf{PROTEINS--GraphSAGE}
& \textbf{NCI1--GIN}
& \textbf{NCI1--GraphSAGE} \\
\midrule
\textbf{SS}
& \accwm{73.1 $\pm$ 2.5}{50.0 $\pm$ 0.0}
& \accwm{72.8 $\pm$ 2.6}{51.0 $\pm$ 3.0}
& \accwm{78.7 $\pm$ 1.5}{50.5 $\pm$ 1.1}
& \accwm{75.5 $\pm$ 1.8}{49.5 $\pm$ 1.1} \\
\textbf{COS}
& \accwm{71.0 $\pm$ 3.0}{82.0 $\pm$ 0.0}
& \accwm{70.5 $\pm$ 3.1}{81.5 $\pm$ 1.1}
& \accwm{76.2 $\pm$ 2.1}{83.6 $\pm$ 0.0}
& \accwm{73.1 $\pm$ 2.4}{84.1 $\pm$ 1.1} \\
\textbf{TRIG}
& \accwm{\bestacc{72.8 $\pm$ 2.6}}{\bestwm{95.1 $\pm$ 3.0}}
& \accwm{71.9 $\pm$ 2.8}{94.5 $\pm$ 0.0}
& \accwm{77.8 $\pm$ 1.8}{94.8 $\pm$ 1.1}
& \accwm{74.8 $\pm$ 2.0}{97.4 $\pm$ 1.1} \\
\textbf{NAT}
& \accwm{72.6 $\pm$ 2.7}{94.8 $\pm$ 1.1}
& \accwm{71.8 $\pm$ 2.9}{94.0 $\pm$ 1.1}
& \accwm{77.6 $\pm$ 1.9}{94.3 $\pm$ 1.1}
& \accwm{74.7 $\pm$ 2.1}{96.9 $\pm$ 0.0} \\
\textbf{EXPL}
& \accwm{72.4 $\pm$ 2.8}{90.6 $\pm$ 0.0}
& \accwm{71.7 $\pm$ 3.0}{89.8 $\pm$ 0.0}
& \accwm{77.9 $\pm$ 1.9}{91.4 $\pm$ 0.0}
& \accwm{74.9 $\pm$ 2.0}{92.2 $\pm$ 0.0} \\
\textbf{OURS}
& \ourscol \accwm{72.5 $\pm$ 2.6}{89.8 $\pm$ 0.0}
& \ourscol \accwm{\bestacc{72.6 $\pm$ 2.6}}{\bestwm{95.8 $\pm$ 4.5}}
& \ourscol \accwm{\bestacc{78.3 $\pm$ 1.6}}{\bestwm{97.7 $\pm$ 0.0}}
& \ourscol \accwm{\bestacc{75.2 $\pm$ 1.8}}{\bestwm{98.2 $\pm$ 1.1}} \\
\bottomrule
\end{tabular}
\end{table*}

\noindent\textbf{Appendix analysis.}
(\emph{i}) \textbf{WM-ACC.}
Across the full 13 dataset--backbone settings, InvGNN-WM achieves the strongest
watermark fidelity in most reported settings, with a single notable
exception on PROTEINS--GIN where TRIG attains higher WM-ACC.
InvGNN-WM remains competitive on PROTEINS and reaches the best WM-ACC among
non-SS methods on PROTEINS--GraphSAGE.
(\emph{ii}) \textbf{Task ACC.}
Across both node-level and graph-level regimes, InvGNN-WM preserves task utility
and stays close to the best non-SS accuracies, indicating that the calibrated
$\beta_{\text{wm}}$ introduces minimal degradation.
These tables compare observed task and watermark accuracies at the configured weights; their conclusions concern those settings and do not imply a causal relation between expressivity and watermark robustness.

\section{Additional comparisons and diagnostics}
\label{app:rebuttal-diagnostics}

We report representative comparisons, local applicability checks, representation probes,
carrier-construction ablations, and resource measurements. Unless specified otherwise, model-training studies use
seeds $41/42/43$; intervals, when supplied, are $95\%$ confidence intervals.

\subsection{Representative comparisons and diagnostics}
\label{app:original-comparisons}

\begin{table}[ht]
\centering
\caption{\textbf{Benchmark comparison (representative backbones).}
Each cell shows \emph{test} \textbf{Task ACC} (\%) on the first line and
\textbf{WM-ACC on private carriers} (\%), abbreviated as \textbf{W-A} on the second line (mean $\pm$95\% CI; three seeds).
Best Task ACC per row \emph{excluding} SS is in \textbf{bold}; the
\bestwm{best WM-ACC} per row is teal bold.
Full results for all dataset--backbone settings are in Appendix~\ref{app:full-main-comparison}.}
\label{tab:main_comparison_compact}

\setlength{\tabcolsep}{3.4pt}
\renewcommand{\arraystretch}{1.05}
\scriptsize

\begin{tabular}{lcccccc}
\toprule
\textbf{Dataset--Backbone} & \textbf{SS} & \textbf{COS} & \textbf{TRIG} & \textbf{NAT} & \textbf{EXPL} & \textbf{OURS} \\
\midrule
Cora--GCN
& \accwm{87.2 $\pm$ 0.8}{49.5 $\pm$ 1.1}
& \accwm{85.5 $\pm$ 1.2}{86.2 $\pm$ 1.1}
& \accwm{86.8 $\pm$ 0.9}{97.7 $\pm$ 3.4}
& \accwm{86.9 $\pm$ 0.9}{96.4 $\pm$ 1.1}
& \accwm{86.7 $\pm$ 1.0}{93.0 $\pm$ 0.0}
& \ourscol \accwm{\bestacc{87.0 $\pm$ 0.8}}{\bestwm{99.0 $\pm$ 1.1}} \\

PubMed--GCN
& \accwm{88.6 $\pm$ 0.9}{49.7 $\pm$ 4.5}
& \accwm{86.4 $\pm$ 1.4}{87.5 $\pm$ 0.0}
& \accwm{87.9 $\pm$ 1.0}{96.9 $\pm$ 0.0}
& \accwm{87.8 $\pm$ 1.1}{96.6 $\pm$ 1.1}
& \accwm{85.7 $\pm$ 1.5}{94.3 $\pm$ 10.7}
& \ourscol \accwm{\bestacc{88.1 $\pm$ 1.0}}{\bestwm{98.7 $\pm$ 1.1}} \\

AmazonPhoto--GCN
& \accwm{91.3 $\pm$ 0.6}{49.2 $\pm$ 0.0}
& \accwm{89.5 $\pm$ 1.1}{88.3 $\pm$ 0.0}
& \accwm{90.8 $\pm$ 0.7}{97.9 $\pm$ 1.1}
& \accwm{90.7 $\pm$ 0.8}{97.4 $\pm$ 1.1}
& \accwm{90.9 $\pm$ 0.8}{94.8 $\pm$ 1.1}
& \ourscol \accwm{\bestacc{91.1 $\pm$ 0.6}}{\bestwm{99.2 $\pm$ 0.0}} \\

\midrule
PROTEINS--GIN
& \accwm{73.1 $\pm$ 2.5}{50.0 $\pm$ 0.0}
& \accwm{71.0 $\pm$ 3.0}{82.0 $\pm$ 0.0}
& \accwm{\bestacc{72.8 $\pm$ 2.6}}{\bestwm{95.1 $\pm$ 3.0}}
& \accwm{72.6 $\pm$ 2.7}{94.8 $\pm$ 1.1}
& \accwm{72.4 $\pm$ 2.8}{90.6 $\pm$ 0.0}
& \ourscol \accwm{72.5 $\pm$ 2.6}{89.8 $\pm$ 0.0} \\

NCI1--GIN
& \accwm{78.7 $\pm$ 1.5}{50.5 $\pm$ 1.1}
& \accwm{76.2 $\pm$ 2.1}{83.6 $\pm$ 0.0}
& \accwm{77.8 $\pm$ 1.8}{94.8 $\pm$ 1.1}
& \accwm{77.6 $\pm$ 1.9}{94.3 $\pm$ 1.1}
& \accwm{77.9 $\pm$ 1.9}{91.4 $\pm$ 0.0}
& \ourscol \accwm{\bestacc{78.3 $\pm$ 1.6}}{\bestwm{97.7 $\pm$ 0.0}} \\
\bottomrule
\end{tabular}
\end{table}

\begin{figure}[t]
\centering
\includegraphics[width=0.83\linewidth]{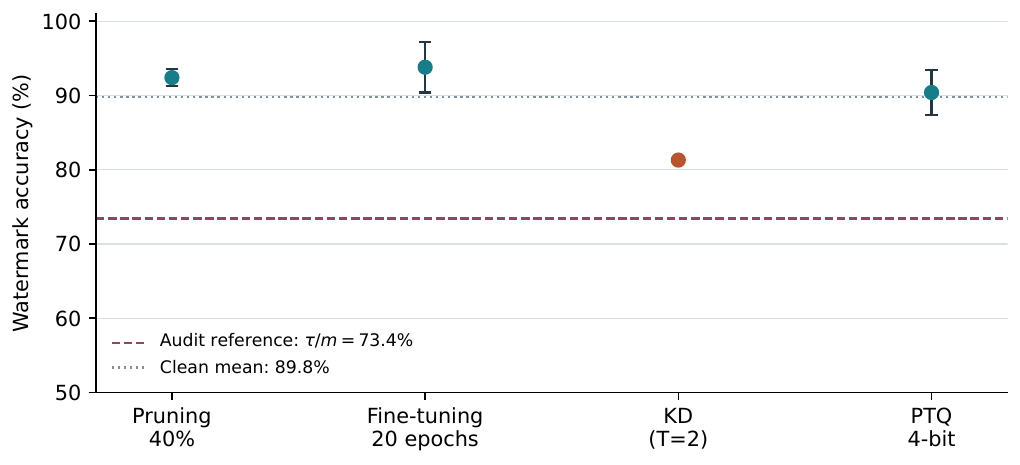}
\caption{\textbf{Common-edit watermark accuracy on PROTEINS--GIN.} Points and error bars show the four edited conditions in Table~\ref{tab:robustness_budget}, using the same means and nominal 95\% Student $t$ intervals as the cross-method comparison. The dashed line is the threshold-equivalent reference $\tau/m=73.4375\%$; all three seeds pass each displayed edit.}
\label{fig:robustness_curve}
\end{figure}

The benchmark comparison spans 13 dataset--backbone settings (the complete table is in
Appendix~\ref{app:full-main-comparison}). Node-level WM-ACC is high across the evaluated
backbones, while graph-level results vary more across methods. PROTEINS--GIN is a clear
exception to any universal superiority claim. Appendix~\ref{app:pl-sensitivity} reports local applicability diagnostics for the
conditional utility bound; fitting local constants does not establish a global PL condition.

\begin{table}[t]
\centering
\caption{\textbf{Match-count diagnostic.} Means, count-CI half-widths, and gaps are rounded to integers. The pooled-null threshold is
$\tau^\ast=94$. In $10^7$ null trials no false positives were observed; the zero-event
resolution is $10^{-7}$. This observation is not a population upper confidence bound.}
\label{tab:uniq_strip}
\small
\begin{tabular}{lccc}
\toprule
Dataset--backbone & Match count $T$ & Threshold $\tau^\ast$ & Mean gap \\
\midrule
PROTEINS--GIN & $115\pm0$ & 94 & $+21$ \\
NCI1--GIN & $125\pm0$ & 94 & $+31$ \\
Cora--GCN & $127\pm1$ & 94 & $+33$ \\
\bottomrule
\end{tabular}
\end{table}

The owner counts exceed the threshold by 21--33 bits in the displayed settings.
Population Type-I control remains conditional on the null and dependence assumptions;
repeated or adaptively selected tests require an appropriate correction.

\subsection{Local PL-sensitivity and applicability checks}
\label{app:pl-sensitivity}

The local sensitivity study in Table~\ref{tab:pl_sensitivity} uses the normalized weight coordinate
$\bar\beta=\beta_{\mathrm{wm}}/\beta_{\max}$. In the executable specification,
\texttt{wm\_weight} instead directly multiplies $\mathcal L_{\mathrm{wm}}$ and is an
absolute coefficient; its numeric value alone neither supplies $\beta_{\max}$ nor identifies
a normalized setting. The neighborhood ratio is the reported checkpoint distance divided
by the local reference radius; values above one lie outside that region. The bound column
gives plug-in reference values for the indicated local regime. Finite loss/gradient ratios,
sensitivity measurements, and these reference values do not certify the theorem's PL or Lipschitz assumptions.

\begin{table}[t]
\centering
\caption{\textbf{PL-sensitivity on PROTEINS--GIN.} ``Res.'' denotes residual connections.
The final column gives Task ACC / WM-ACC (\%). ``Bound'' is a plug-in reference value;
a dash denotes a setting outside the reported local regime, not a zero observed loss change.}
\label{tab:pl_sensitivity}
\small
\setlength{\tabcolsep}{4pt}
\begin{tabular}{lccccc}
\toprule
Layers / Res. / epochs / $\bar\beta$ & $\widehat\mu_{\mathrm{PL}}$ & Ratio & Bound & $\Delta\mathcal L_{\mathrm{task}}$ & Task / WM-ACC \\
\midrule
4 / no / 100 / 0.8 & 0.42 & 0.61 & 0.0128 & 0.009 & 72.5 / 89.8 \\
8 / no / 100 / 0.8 & 0.16 & 1.09 & -- & 0.026 & 70.9 / 85.2 \\
8 / yes / 100 / 0.8 & 0.30 & 0.78 & 0.0128 & 0.012 & 72.0 / 88.5 \\
4 / yes / 300 / 0.8 & 0.31 & 1.02 & -- & 0.018 & 71.7 / 84.9 \\
4 / yes / 100 / 2.0 & 0.39 & 1.34 & -- & 0.061 & 69.8 / 96.1 \\
\bottomrule
\end{tabular}
\end{table}

In the reported 8-layer comparison, the residual model has a neighborhood ratio below
one, whereas the non-residual model lies outside the reference region. The longer-training
and larger-weight settings also lie outside it. The carrier statistic
$\widehat\rho_0=7.6\times10^{-4}$ and edit-drift measurements are finite diagnostics;
applying the concentration or drift result still requires its stated assumptions.

\subsection{Shared-backbone invariant information}
\label{app:representation-coupling}

Table~\ref{tab:representation_coupling} reports four representation variants. Only full
InvGNN-WM updates the shared task backbone through the watermark loss. Stop-gradient blocks
these updates, the separate watermark encoder has no shared watermark representation, and
task-only receives no watermark training. For each frozen task backbone, a fresh ridge
readout with regularization parameter $1$ fits normalized targets $\tilde\lambda_2$ on the
first $80\%$ of a disjoint pseudo-carrier pool in generation order. Probe $R^2$ uses its
unclipped predictions on the remaining $20\%$. The same fitted readout supplies the fresh-head
WM-ACC: predictions on owner carriers are clipped to $[0,1]$ and decoded at $0.5$.
Owner carriers are excluded from readout fitting. This measures recoverable invariant
information; it does not establish that task predictions causally depend on $\lambda_2$.

\subsection{Carrier-filter and size-cap ablations}
\label{app:carrier-filter-ablation}

Weisfeiler--Lehman (WL) non-collision excludes observed training-graph collisions;
Kolmogorov--Smirnov (KS) tests match the chosen local and spectral statistics.
The default size constraint is $n\le\max\{4,\min\{32,\lfloor n_{0.25}\rfloor\}\}$; the no-cap ablation removes this joint constraint.
All rows in Table~\ref{tab:carrier_filter_ablation} retain $m=128$ and $\tau=94$.
Carrier-vs-task ACC measures a classifier's ability to distinguish carriers from task graphs;
$p$ is the empirical fraction of one-bits in the owner key. The false-positive rate (FPR)
column reports finite-null observations for each construction. Small empirical rates and
zero-event display limits do not by themselves establish population guarantees.

\begin{table}[t]
\centering
\caption{\textbf{Carrier construction on PROTEINS--GIN.}
The $<10^{-7}$ entries express finite-null display limits rather than population upper confidence bounds.
The Erdos--Renyi row changes our carrier construction; it does not evaluate the security
of PreGIP, which uses a different pair-consistency statistic.}
\label{tab:carrier_filter_ablation}
\small
\setlength{\tabcolsep}{4pt}
\begin{tabular}{p{0.34\textwidth}cccc}
\toprule
Construction & \makecell{Carrier-vs-task\\ACC (\%)} & \makecell{Key\\balance $p$} & \makecell{Empirical\\FPR} & Time (s) \\
\midrule
Full WL + local KS + spectral KS & 51.3 & 0.49 & $<10^{-7}$ & 52.4 \\
Without WL & 54.9 & 0.50 & $2.4\times10^{-7}$ & 35.4 \\
Without degree/clustering KS & 69.8 & 0.57 & $1.8\times10^{-5}$ & 21.7 \\
Without spectral-radius KS & 60.7 & 0.55 & $4.6\times10^{-6}$ & 28.9 \\
Erdos--Renyi carriers & 99.1 & 0.79 & $8.4\times10^{-2}$ & 1.4 \\
Full protocol without size cap & 52.6 & 0.50 & $<10^{-7}$ & 176.3 \\
\bottomrule
\end{tabular}
\end{table}

WL mainly prevents empirical collisions. Local-statistic matching gives the largest change
in detectability, spectral matching further improves the measured calibration, and the size
cap mainly limits cost. High raw WM-ACC on easily detected carriers is therefore insufficient
for a reliable audit. These are finite-classifier and finite-null diagnostics, not proof of
cryptographic indistinguishability.

\subsection{OGBG-MOLHIV, GPS, and compute}
\label{app:scale-compute}

Table~\ref{tab:ogb_compute} in the main text reports OGBG-MOLHIV results for GIN and GPS,
including the task ROC-AUC, detection metrics, and per-run resource measurements.

\paragraph{Independent-training detection protocol.}
\label{app:ogb-independent-detection}
For each backbone, each master seed in $\{41,42,43\}$ derives 100 independent
pairs of training and carrier-generation sub-seeds. Each run generates 128 private
carriers and their matching 128-bit owner key from normalized algebraic connectivity
using the invariant-based encoding, then independently initializes and trains a new
checkpoint with the task-logit watermark loss. The dataset split, architecture, and
hyperparameters remain fixed. The architecture-specific decoder is determined before
training, and the calibrated detection threshold remains fixed across runs.
After training, the matching checkpoint, owner key, carrier set, and decoder are frozen
for one deterministic complete audit. The correct-bit count $T$ determines success
through $T\ge\tau$, with $\tau=94$ at the target $\alpha=10^{-6}$.
One independently trained watermarked model with its matching owner audit material is
one positive trial. The study contains 300 such trials per backbone and 600 in total.
For master seeds $41/42/43$, the success counts are $95/96/97$ for GIN and
$97/98/99$ for GPS, with 100 trials per master seed.

\paragraph{Detection confidence intervals.}
Let $K$ be the number of successful positive trials among the $N=300$ independent
training runs for one backbone, and let $\hat p=K/N$ be the pooled detection rate.
For the standard two-sided $95\%$ Wilson interval without continuity correction,
let $z=\Phi^{-1}(0.975)$, where $\Phi$ is the standard normal cumulative distribution
function. The lower endpoint $L$ uses the minus sign and the upper endpoint $U$
uses the plus sign in
\[
  L,U=\frac{\hat p+z^2/(2N)\mp z\sqrt{\hat p(1-\hat p)/N+z^2/(4N^2)}}
  {1+z^2/N}.
\]
For GIN, the pooled rate is $288/300=0.96$, with interval $[0.9314,0.9770]$.
For GPS, the pooled rate is $294/300=0.98$, with interval $[0.9571,0.9908]$.
The intervals describe variation across independently generated and trained watermark
instances under the fixed dataset split, architecture, hyperparameters, decoder, and
calibrated threshold. The other columns of Table~\ref{tab:ogb_compute} retain the
three-seed benchmark reporting.

\paragraph{Resource-accounting scope.}
The four benchmark configurations evaluated at three seeds give the training-only subtotal
$3(37.3+42.8+66.3+75.3)/60=11.085\approx11.1$ A100 GPU-hours.
For the separate detection study, extrapolating the mean watermarked-run times of $42.8$
minutes for GIN and $75.3$ minutes for GPS to 300 positive models of each backbone gives
\[
300\times42.8/60+300\times75.3/60=590.5\ \text{A100 GPU-hours}.
\]
This estimate covers only the 600 positive-model training runs under the same hardware and
training conditions as the per-run measurements. It is not a measured whole-study total
or a lower bound. Null-model training, carrier construction, calibration, attacks, other
benchmarks, hyperparameter search, and preliminary or failed runs are outside this estimate.
The 11.1-hour benchmark subtotal is reported separately and is not added to the extrapolation.

Owner-side invariant preprocessing is separate from model-query latency. The
PROTEINS carrier study reports $0.08$\,ms per exact eigenvalue computation and $0.05$\,ms
for the 10-step approximation, with an approximately $1.0$-point WM-ACC difference
(Table~\ref{tab:approx_lambda2}). Those preprocessing times do not measure an API audit.
The additional evidence covers one OGB molecular dataset and one global-attention architecture.
Heterogeneous, temporal, higher-order, industrial, sensitive-domain, and production-API graph
settings remain untested. The task-logit evaluation is graph-level; its node-classification
and hard-label-only versions remain untested.

\subsection{Matched-comparison protocol and interpretation}
\label{app:pregip-protocol}

PreGIP-adapted uses supervised task training with 64 synthetic graph pairs, giving the
same 128-graph query budget as InvGNN-WM's 128 individual carriers. Each key pair is formed
by two rewired graphs from a common source. For watermark-minibatch embeddings
$u_1,\ldots,u_{2B}$, watermark mean $\bar u_W$, and task-minibatch mean $\bar u_T$,
the implemented auxiliary objective is
\[
\mathcal L_{\mathrm{pair}}=
\frac{1}{B}\sum_{j=1}^{B}\bigl(1-\cos(u_{2j-1},u_{2j})\bigr)
+0.1\bigl[\cos(\bar u_W,\bar u_T)-0.1\bigr]_+,
\]
where $[x]_+=\max\{x,0\}$. Joint training minimizes the supervised task loss plus
the weighted auxiliary objective. This paired-source supervised adaptation is identified
separately from PreGIP's original pretraining algorithm. The backbone, task split,
optimizer, training epochs, validation checkpoint selection, and query budget are matched;
method-specific watermark objectives and key construction remain different.

PreGIP's verification score is the mean cosine similarity over its 64 key pairs. A score
at least $0.8000$ on PROTEINS--GIN or $0.8200$ on NCI1--GIN passes the corresponding
detector. These score cutoffs and InvGNN-WM's match-count thresholds are calibrated
separately on disjoint pooled-null data before attacks and remain fixed afterwards.
The nominal calibration target is $\alpha=10^{-6}$; the empirical null diagnostics have
the finite-sample interpretation described above. AUROC summarizes discrimination across
score thresholds. TPR uses Wilson intervals; Appendix~\ref{app:matched-detection-ci}
gives the owner-instance protocol, all pooled counts and intervals, and the AUROC intervals.

The auxiliary and task-logit interfaces must be assessed separately. Output-only KD produces
a student with only the ordinary task interface. An auxiliary-head audit is then unavailable,
which is an interface failure rather than a measured WM-ACC or TPR of zero. The task-logit
variant applies $d(z)=\sigma(a^\top z+b)$ to ordinary logits $z\in\mathbb R^C$ using fixed
owner-side audit coefficients $a,b$ set before training and attack evaluation. Deployment
requires the task logits and these coefficients. No test or attack outcome selects the decoder. The watermark objective,
carriers, invariant targets, owner key, and match-count audit otherwise stay the same.
This construction provides no interface advantage when the API returns only hard labels.

\subsection{Software versions and asset sources}
\label{app:assets}

Table~\ref{tab:software_assets} lists the software versions recorded in the experimental
environment and the licenses stated by their release metadata or official documentation.
Package names link to the corresponding source records. These package licenses describe
software; they do not automatically determine the licenses of datasets loaded by that software.

\begin{table}[t]
\centering
\caption{\textbf{Recorded software assets and upstream licenses.}
Versions are those reported for the experimental environment; links identify the upstream
release or versioned documentation. CUDA is subject to NVIDIA's license agreement.}
\label{tab:software_assets}
\small
\setlength{\tabcolsep}{6pt}
\begin{tabular}{lll}
\toprule
Software / source record & Recorded version & Stated license \\
\midrule
\href{https://pypi.org/project/torch/2.2.2/}{PyTorch} & 2.2.2 & BSD-3-Clause \\
\href{https://pypi.org/project/torch-geometric/2.5.3/}{PyTorch Geometric} & 2.5.3 & MIT \\
\href{https://pypi.org/project/torch-scatter/2.1.2/}{torch-scatter} & 2.1.2 & MIT \\
\href{https://pypi.org/project/torch-sparse/0.6.18/}{torch-sparse} & 0.6.18 & MIT \\
\href{https://pypi.org/project/torch-cluster/1.6.3/}{torch-cluster} & 1.6.3 & MIT \\
\href{https://pypi.org/project/torch-spline-conv/1.2.2/}{torch-spline-conv} & 1.2.2 & MIT \\
\href{https://numpy.org/doc/1.26/license.html}{NumPy} & 1.26.4 & BSD-3-Clause \\
\href{https://pypi.org/project/scikit-learn/1.4.2/}{scikit-learn} & 1.4.2 & BSD-3-Clause \\
\href{https://pypi.org/project/networkx/3.2.1/}{NetworkX} & 3.2.1 & BSD-3-Clause \\
\href{https://docs.nvidia.com/cuda/archive/12.1.0/eula/index.html}{NVIDIA CUDA Toolkit} & 12.1 & NVIDIA CUDA Toolkit EULA \\
\bottomrule
\end{tabular}
\end{table}

Dataset sources are the citation-network/Planetoid benchmarks~\citep{Sen2008CoraPubMed,Yang2016Planetoid},
the Amazon-Photo graph benchmark~\citep{Shchur2019Pitfalls},
TUDataset for PROTEINS/NCI1~\citep{Morris2020TUDataset}, and
OGB for OGBG-MOLHIV~\citep{hu2020open}. The
\href{https://ogb.stanford.edu/docs/graphprop/}{official OGB graph-property documentation}
lists the OGB molecular datasets under the MIT license.
The evaluated architectures and baseline methods are credited in
Section~\ref{sec:exp_setup} and Appendix~\ref{sec:related_work};
PreGIP-adapted is explicitly identified as an adaptation of~\citet{Dai2025PreGIP}.

\clearpage
\subsection{Owner-instance detection and AUROC intervals for the matched comparison}
\label{app:matched-detection-ci}

Each setting in Table~\ref{tab:pregip_matched} uses 300 independent owner instances.
Clean, FT-100, and available KD outcomes are paired within instance; each condition
contributes one binary result per owner. Detectors are calibrated before evaluation
and remain fixed across attacks (Appendix~\ref{app:pregip-protocol}).

Pool the success counts $K$ over seed groups 41, 42, and 43 with $N=300$.
Table~\ref{tab:matched_detection_ci} reports $\hat p=K/N$ and two-sided $95\%$ Wilson
intervals without continuity correction, as defined in Appendix~\ref{app:scale-compute}.
These are marginal condition-specific intervals, not simultaneous or paired-difference
intervals; pairing adds no independent trials. Auxiliary-head KD is unavailable when
the scalar endpoint is absent. Task ACC follows the separate three-seed protocol.

\begin{table}[htbp]
\centering
\caption{\textbf{Pooled TPR counts and Wilson intervals for Table~\ref{tab:pregip_matched}.}
Each cell gives $K/300$ on the first line and its two-sided $95\%$ Wilson interval on
the second. All backbones are GIN. A dash denotes an unavailable audit.}
\label{tab:matched_detection_ci}
\small
\setlength{\tabcolsep}{4pt}
\renewcommand{\arraystretch}{1.25}
\begin{tabular}{llccc}
\toprule
Dataset & Method & Clean & FT-100 & KD \\
\midrule
PROTEINS & PreGIP-adapted & \shortstack{$252/300$\\$[0.7943,0.8771]$} & \shortstack{$228/300$\\$[0.7086,0.8048]$} & \shortstack{$183/300$\\$[0.5537,0.6635]$} \\
& InvGNN-WM, auxiliary & \shortstack{$297/300$\\$[0.9710,0.9966]$} & \shortstack{$264/300$\\$[0.8383,0.9120]$} & -- \\
& InvGNN-WM, task-logit & \shortstack{$294/300$\\$[0.9571,0.9908]$} & \shortstack{$276/300$\\$[0.8837,0.9457]$} & \shortstack{$222/300$\\$[0.6876,0.7864]$} \\
\midrule
NCI1 & PreGIP-adapted & \shortstack{$270/300$\\$[0.8608,0.9291]$} & \shortstack{$249/300$\\$[0.7834,0.8683]$} & \shortstack{$201/300$\\$[0.6149,0.7208]$} \\
& InvGNN-WM, auxiliary & \shortstack{$300/300$\\$[0.9874,1.0000]$} & \shortstack{$288/300$\\$[0.9314,0.9770]$} & -- \\
& InvGNN-WM, task-logit & \shortstack{$300/300$\\$[0.9874,1.0000]$} & \shortstack{$291/300$\\$[0.9440,0.9841]$} & \shortstack{$246/300$\\$[0.7726,0.8593]$} \\
\bottomrule
\end{tabular}
\end{table}

\paragraph{Clean detection AUROC.}
Table~\ref{tab:matched_auroc_ci} gives the six clean AUROC estimates and their
$95\%$ confidence intervals. The statistic pools positive and null detector scores,
with larger scores indicating stronger watermark evidence. The hierarchical paired-owner
percentile bootstrap uses 2,000 replicates: it first resamples the three master-seed
groups with replacement, then resamples owner instances within each selected group.
Each owner's positive and null scores remain paired in the resample. Pooled AUROC is
recomputed for each replicate, and its $2.5$th and $97.5$th percentiles define the interval.
With only three master-seed groups, between-group variability is estimated from a small
number of groups. These marginal AUROC intervals are separate from the Wilson intervals
for condition-specific TPR above and do not establish significance of method differences.

\begin{table}[htbp]
\centering
\caption{\textbf{Clean detection AUROC and $95\%$ intervals for Table~\ref{tab:pregip_matched}.}
All backbones are GIN. Intervals use the hierarchical paired-owner percentile bootstrap
with 2,000 replicates.}
\label{tab:matched_auroc_ci}
\small
\setlength{\tabcolsep}{6pt}
\renewcommand{\arraystretch}{1.1}
\begin{tabular}{llcc}
\toprule
Dataset & Method & AUROC & $95\%$ CI \\
\midrule
PROTEINS & PreGIP-adapted & $0.989$ & $[0.9828,0.9946]$ \\
& InvGNN-WM, auxiliary & $0.998$ & $[0.9950,1.0000]$ \\
& InvGNN-WM, task-logit & $0.997$ & $[0.9916,1.0000]$ \\
\midrule
NCI1 & PreGIP-adapted & $0.993$ & $[0.9883,0.9972]$ \\
& InvGNN-WM, auxiliary & $0.999$ & $[0.9955,1.0000]$ \\
& InvGNN-WM, task-logit & $0.999$ & $[0.9956,1.0000]$ \\
\bottomrule
\end{tabular}
\end{table}

\end{document}